\documentclass[12pt]{article}

\usepackage[numbers, compress]{natbib}
\usepackage[utf8]{inputenc} 
\usepackage[T1]{fontenc}    
\usepackage{hyperref}       
\usepackage{url}            
\usepackage{booktabs}       
\usepackage{amsfonts}       
\usepackage{nicefrac}       
\usepackage{microtype}      
\usepackage{xcolor}         

\usepackage{amsmath, nccmath, amssymb}
\usepackage{amsthm}
\usepackage[ruled]{algorithm2e}
\usepackage{cleveref}
\usepackage{graphicx}
\usepackage{tcolorbox}
\usepackage{mathrsfs}
\usepackage{graphbox}
\usepackage{caption}
\usepackage{subcaption}
\usepackage{wrapfig}
\usepackage{float}
\usepackage{enumitem}
\usepackage{tabu}
\usepackage{doi}
\usepackage{tikz}
\usepackage[mode=buildnew]{standalone}
\usepackage{pgfplots}

\newtheorem{theorem}{Theorem}
\newtheorem{lemma}{Lemma}
\newtheorem{assumption}{Assumption}
\newtheorem{definition}{Definition}
\newtheorem{proposition}{Proposition}
\newtheorem{corollary}{Corollary}

\DeclareMathOperator{\tr}{tr}
\DeclareMathOperator{\poly}{poly}

\newtcolorbox{boxedeq}[1][]{colback=white, sharp corners, fonttitle=\bfseries, title=#1, center title}
\newcommand*\cvec[2]{\begin{pmatrix} #1 \\ #2\end{pmatrix}}
\newcommand{\norm}[1]{\left\lVert #1 \right\rVert}
\newcommand{\abs}[1]{\left\lvert #1 \right\rvert}
\newcommand{\qqtext}[1]{\qquad\text{#1}\qquad}
\newcommand{\eval}[1]{\left.#1\right\rvert}
\newcommand{\E}{\mathbb{E}}
\newcommand{\jnote}[1]{}
\newcommand{\anote}[1]{}
\newcommand{\tnote}[1]{}

\ifdefined\usebigfont

\usepackage{times}
\usepackage[fontsize=13pt]{scrextend}
\usepackage[letterpaper,left=1.56in,right=1.56in,top=1.71in,bottom=1.77in]{geometry}
\newgeometry{letterpaper,left=1.56in,right=1.56in,top=1.71in,bottom=1.77in}
\else
\usepackage{times}
\usepackage[margin=1.25in]{geometry}
\fi

\begin{document}

\setlength{\parindent}{0pt}
\setlength{\parskip  }{5.5pt}

\title{Label Noise SGD Provably Prefers Flat Global Minimizers}

\author{%
  Alex Damian \\
  Princeton University\\
  \texttt{ad27@princeton.edu}
  \and
  Tengyu Ma \\
  Stanford University \\
  \texttt{tengyuma@stanford.edu} \\
  \and
  Jason D. Lee \\
  Princeton University \\
  \texttt{jasonlee@princeton.edu}
}

\maketitle

\begin{abstract}
In overparametrized models, the noise in stochastic gradient descent (SGD) implicitly regularizes the optimization trajectory and determines which local minimum SGD converges to. Motivated by empirical studies that demonstrate that training with noisy labels improves generalization, we study the implicit regularization effect of SGD with label noise. We show that SGD with label noise converges to a stationary point of a regularized loss $L(\theta) +\lambda R(\theta)$, where $L(\theta)$ is the training loss, $\lambda$ is an effective regularization parameter depending on the step size, strength of the label noise, and the batch size, and $R(\theta)$ is an explicit regularizer that penalizes sharp minimizers. Our analysis uncovers an additional regularization effect of large learning rates beyond the linear scaling rule that penalizes large eigenvalues of the Hessian more than small ones. We also prove extensions to classification with general loss functions, SGD with momentum, and SGD with general noise covariance, significantly strengthening the prior work of \citet{blanc2019implicit} to global convergence and large learning rates and of \citet{haochen2020shape} to general models.
\end{abstract}

\section{Introduction}
One of the central questions in modern machine learning theory is the generalization capability of overparametrized models trained by stochastic gradient descent (SGD). Recent work identifies the implicit regularization effect due to the optimization algorithm as one key factor in explaining the generalization of overparameterized models~\citep{soudry2018implicit,gunasekar2018characterizing,li2017algorithmic,gunasekar2017implicit}. This implicit regularization is controlled by many properties of the optimization algorithm including search direction~\citep{gunasekar2018characterizing}, learning rate~\citep{li2019towards}, batch size~\citep{smith2017don}, momentum~\citep{liu2019bad} and dropout~\citep{mianjy2018implicit}.

The parameter-dependent noise distribution in SGD is a crucial source of regularization~\citep{keskar2016large,lecun2012efficient}. \citet{blanc2019implicit} initiated the study of the regularization effect of label noise SGD with square loss\footnote{Label noise SGD computes the stochastic gradient by first drawing a sample $(x_i,y_i)$, perturbing $y'_i= y_i+\epsilon$ with $\epsilon \sim \{-\sigma,\sigma\}$, and computing the gradient with respect to $(x_i,y'_i)$.} by characterizing the local stability of global minimizers of the training loss. By identifying a data-dependent regularizer $R(\theta)$, \citet{blanc2019implicit} proved that label noise SGD locally diverges from the global minimizer $\theta^\ast$ if and only if $\theta^\ast$ is not a first-order stationary point of $$\min_\theta R(\theta) \text{ subject to } L(\theta) = 0.$$ The analysis is only able to demonstrate that with sufficiently small step size $\eta$, label noise SGD initialized at $\theta^\ast$ locally diverges by a distance of $\eta^{0.4}$ and correspondingly decreases the regularizer by $\eta^{0.4}$. This is among the first results that establish that the noise distribution alters the local stability of stochastic gradient descent. 
However, the parameter movement of $\eta^{0.4}$ is required to be inversely polynomially small in dimension and condition number and is thus  too small to affect the predictions of the model.

\citet{haochen2020shape}, motivated by the local nature of \citet{blanc2019implicit}, analyzed label noise SGD in the quadratically-parametrized linear regression model~\citep{vaskevicius2019implicit,woodworth2020kernel,moroshko2020implicit}. Under a well-specified sparse linear regression model and with isotropic features, \citet{haochen2020shape} proved that label noise SGD recovers the sparse ground-truth despite overparametrization, which demonstrated a global implicit bias towards sparsity in the quadratically-parametrized linear regression model.

This work seeks to identify the global implicit regularization effect of label noise SGD. Our primary result, which supports \citet{blanc2019implicit}, proves that label noise SGD converges to a stationary point of $L(\theta) +\lambda R(\theta)$, where the regularizer $R(\theta)$ penalizes sharp regions of the loss landscape.

The focus of this paper is on label noise SGD due to its strong regularization effects in both real and synthetic experiments~\citep{shallue2018measuring,szegedy2016rethinking,wen2019interplay}. Furthermore, label noise is used in large-batch training as an additional regularizer~\citep{shallue2018measuring} when the regularization from standard regularizers (e.g. mini-batch, batch-norm, and dropout) is not sufficient. Label noise SGD is also known to be less sensitive to initialization, as shown in \citet{haochen2020shape}. 
In stark contrast, mini-batch SGD remains stuck when initialized at any poor global minimizer. Our analysis demonstrates a global regularization effect of label noise SGD by proving it converges to a stationary point of a regularized loss $ L(\theta) +\lambda R(\theta)$, even when initialized at a zero error global minimum.

The learning rate and minibatch size in SGD are also known to be important sources of regularization~\citep{goyal2017accurate}.
Our main theorem highlights the importance of learning rate and batch size as the hyperparameters that control the balance between the loss and the regularizer -- larger learning rate and smaller batch size leads to stronger regularization.

\Cref{sec:setup} reviews the notation and assumptions used throughout the paper. \Cref{sec:main} formally states the main result and \Cref{sec:sketch} sketches the proof. \Cref{sec:experiments} presents experimental results which support our theory. Finally, \Cref{sec:discussion} discusses the implications of this work.

\section{Problem Setup and Main Result}
\label{sec:setup}

\Cref{sec:notation} describes our notation and the SGD with label noise algorithm. \Cref{sec:regularizer} introduces the explicit formula for the regularizer $R(\theta)$. \Cref{sec:epsgammasp,sec:main} formally state our main result.

\subsection{Notation}\label{sec:notation}

We focus on the regression setting (see \Cref{sec:classification} for the extension to the classification setting). Let $\{(x_i,y_i)\}_{i \in [n]}$ be $n$ datapoints with $x_i \in \mathcal{D}$ and $y_i \in \mathbb{R}$. Let $f:\mathcal{D} \times \mathbb{R}^d \to \mathbb{R}$ and let $f_i(\theta) = f(x_i,\theta)$ denote the value of $f$ on the datapoint $x_i$. Define $\ell_i(\theta) = \frac{1}{2}\left(f_i(\theta)-y_i\right)^2$ and $L(\theta) = \frac{1}{n} \sum_{i=1}^n \ell_i(\theta)$.
Then we will follow \Cref{alg:sgdln} which adds fresh additive noise to the labels $y_i$ at every step before computing the gradient:

\begin{algorithm}[H]
\SetAlgoLined
\KwIn{$\theta_0$, step size $\eta$, noise variance $\sigma^2$, batch size $B$, steps $T$}
\For{$k = 0$ to $T-1$}
{
Sample batch $\mathcal{B}^{(k)} \subset [n]^{B}$ uniformly and label noise $\epsilon^{(k)}_i \sim \{-\sigma,\sigma\}$ for $i \in \mathcal{B}^{(k)}$. \\
Let $\hat{\ell}_i^{(k)}(\theta) = \frac{1}{2} \left(f_{i}(\theta)-y_{i}-\epsilon^{(k)}_i\right)^2$ and $\hat L^{(k)} = \frac{1}{B} \sum_{i \in \mathcal{B}^{(k)}} \hat{\ell}_i^{(k)}$. \\
$\theta_{k+1} \leftarrow \theta_k - \eta \nabla \hat{L}^{(k)}(\theta_k)$
}
\caption{SGD with Label Noise}
\label{alg:sgdln}
\end{algorithm}

Note that $\sigma$ controls the strength of the label noise and will control the strength of the implicit regularization in \Cref{thm:sgdsp}. Throughout the paper we will use $\|\cdot\| = \|\cdot\|_2$. We make the following standard assumption on $f$:
\begin{assumption}[Smoothness]\label{asm:smooth}
	We assume that each $f_i$ is $\ell_f$-Lipschitz, $\nabla f_i$ is $\rho_f$-Lipschitz, and $\nabla^2 f_i$ is $\kappa_f$-Lipschitz with respect to $\|\cdot\|_2$ for $i = 1,\ldots,n$.
\end{assumption}
We will define $\ell = \ell_f^2$ to be an upper bound on $\|\frac{1}{n} \sum_i \nabla f_i(\theta)\nabla f_i(\theta)^T\|_2$, which is equal to $\|\nabla^2 L(\theta)\|_2$ at any global minimizer $\theta$. Our results extend to any learning rate $\eta \in (0,\frac{2}{\ell})$. However, they do not extend to the limit as $\eta \to \frac{2}{\ell}$. Because we still want to track the dependence on $\frac{1}{\eta}$, we do not assume $\eta$ is a fixed constant and instead assume some constant separation:
\begin{assumption}[Learning Rate Separation]\label{asm:eta}
	There exists a constant $\nu \in (0,1)$ such that $\eta \le \frac{2-\nu}{\ell}$.
\end{assumption}

In addition, we make the following local Kurdyka-\L ojasiewicz assumption (KL assumption) which ensures that there are no regions where the loss is very flat. The KL assumption is very general and holds for some $\delta>0$ for any analytic function defined on a compact domain (see \Cref{lem:klanalytic}).

\begin{assumption}[KL]\label{asm:kl}
	Let $\theta^*$ be any global minimizer of $L$. Then there exist $\epsilon_{KL}>0,\mu > 0$ and $0 < \delta \le 1/2$ such that if $L(\theta) - L(\theta^*) \le \epsilon_{KL}$, then $L(\theta) - L(\theta^*) \le \mu \| \nabla L(\theta) \|^{1+\delta}$.
\end{assumption}

We assume $L(\theta^*) = 0$ for any global minimizer $\theta^*$. Note that if $L$ satisfies \Cref{asm:kl} for some $\delta$ then it also satisfies \Cref{asm:kl} for any $\delta'<\delta$. \Cref{asm:kl} with $\delta = 1$ is equivalent to the much stronger Polyak-\L ojasiewicz condition which is equivalent to local strong convexity.

We will use $O,\Theta,\Omega$ to hide any polynomial dependence on $ \mu, \ell_f, \rho_f, \kappa_f, \nu, 1/\sigma, n, d$ and $\tilde O$ to hide additional polynomial dependence on $\log 1/\eta, \log B$. 

\subsection{The Implicit Regularizer $R(\theta)$}\label{sec:regularizer}

For $L,\sigma^2,B,\eta$ as defined above, we define the implicit regularizer $R(\theta)$, the effective regularization parameter $\lambda$, and the regularized loss $\tilde{L}(\theta)$:
\begin{align}
	R(\theta) = -\frac{1}{2 \eta}\tr \log \left(1-\frac{\eta}{2} \nabla^2 L(\theta)\right), && \lambda = \frac{\eta \sigma^2}{B}, && \tilde{L}(\theta) = L(\theta) + \lambda R(\theta). \label{eq:regdef}
\end{align}

\begin{figure}
	\centering
	\begin{minipage}[c]{0.36\textwidth}
    	\includestandalone[width=\linewidth]{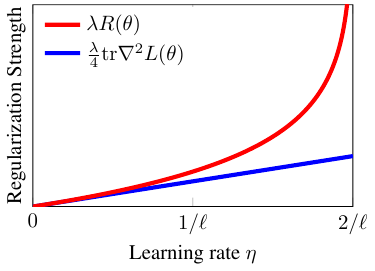}
	\end{minipage}\hfill
	\begin{minipage}[c]{0.58\textwidth}
    \caption{Comparison of regularization strength in one dimension for the implicit regularizer $\lambda R(\theta) \propto \log(1-\frac{\eta\ell}{2})$ and its linear approximation around $\eta = 0$, $\frac{\lambda}{4}\tr \nabla^2 L(\theta) \propto \eta \ell$. Here $\ell = \|\nabla^2 L(\theta)\|_2$ measures the sharpness at $\theta$.}
    \label{fig:regplot}
  	\end{minipage}
\end{figure}

Here $\log$ refers to the matrix logarithm. To better understand the regularizer $R(\theta)$, let $\lambda_1,\ldots,\lambda_d$ be the eigenvalues of $\nabla^2 L(\theta)$ and let $R(\lambda_i) = -\frac{1}{2\eta} \log(1-\frac{\eta\lambda_i}{2})$. Then,
\begin{align*}
	R(\theta) = \sum_{i=1}^d R(\lambda_i) = \sum_{i=1}^d \left(\frac{\lambda_i}{4} + \frac{\eta \lambda_i^2}{16} + \frac{\eta^2 \lambda_i^3}{48} + \ldots\right).
\end{align*}
In the limit as $\eta \to 0$, $R(\theta) \to \frac{1}{4} \tr \nabla^2 L(\theta)$, which matches the regularizer in \citet{blanc2019implicit} for infinitesimal learning rate near a global minimizer. However, in additional to the linear scaling rule, which is implicit in our definition of $\lambda$, our analysis uncovers an \textbf{additional regularization effect of large learning rates} that penalizes larger eigenvalues more than smaller ones (see \Cref{fig:regplot} and \Cref{sec:strengthlargeeta}).

The goal of this paper is to show that Algorithm \ref{alg:sgdln} converges to a stationary point of the regularized loss $\tilde{L}= L + \lambda R$. In particular, we will show convergence to an $(\epsilon,\gamma)$-stationary point, which is defined in the next section.
                          
\subsection{$(\epsilon,\gamma)$-Stationary Points}\label{sec:epsgammasp}

We begin with the standard definition of an approximate stationary point:
\begin{definition}[$\epsilon$-stationary point]
	$\theta$ is an $\epsilon$-stationary point of $f$ if $\|\nabla f(\theta)\| \le \epsilon.$
\end{definition}
In stochastic gradient descent it is often necessary to allow $\lambda = \frac{\eta \sigma^2}{B}$ to scale with $\epsilon$ to reach an $\epsilon$-stationary point~\citep{ge2015escaping,jin2019stochastic} (e.g., $\lambda$ may need to be less than $\epsilon^2$). However, for $\lambda = O(\epsilon)$, any local minimizer $\theta^*$ is an $\epsilon$-stationary point of $\tilde{L} = L + \lambda R$. Therefore, reaching a $\epsilon$-stationary point of $\tilde{L}$ would be equivalent to finding a local minimizer and would not be evidence for implicit regularization. To address this scaling issue, we consider the rescaled regularized loss:
\begin{align*}
	\frac{1}{\lambda} \tilde{L} = \frac{1}{\lambda} L + R.
\end{align*}
Reaching an $\epsilon$-stationary point of $\frac{1}{\lambda} \tilde{L}$ requires non-trivially taking the regularizer $R$ into account. However, it is not possible for \Cref{alg:sgdln} to reach an $\epsilon$-stationary point of $\frac{1}{\lambda} \tilde L$ even in the ideal setting when $\theta$ is initialized near a global minimizer $\theta^*$ of $\tilde{L}$. The label noise will cause fluctuations of order $\sqrt{\lambda}$ around $\theta^*$ (see section \ref{sec:sketch}) so $\|\nabla L\|$ will remain around $\sqrt{\lambda}$. This causes $\frac{1}{\lambda} \nabla L$ to become unbounded for $\lambda$ (and therefore $\epsilon$) sufficiently small, and thus \Cref{alg:sgdln} cannot converge to an $\epsilon$-stationary point. 
We therefore prove convergence to an \textit{$(\epsilon,\gamma)$-stationary point}:
\begin{definition}[$(\epsilon,\gamma)$-stationary point]\label{def:epsgamma}
	$\theta$ is an $(\epsilon,\gamma)$-stationary point of $f$ if there exists some $\theta^*$ such that $\|\nabla f(\theta^*)\|\le \epsilon$ and $\| \theta - \theta^* \| \le \gamma$.
\end{definition}
Intuitively, \Cref{alg:sgdln} converges to an $(\epsilon,\gamma)$-stationary point when it converges to a neighborhood of some $\epsilon$-stationary point $\theta^*$.

\subsection{Main Result}\label{sec:main}
Having defined an $(\epsilon,\gamma)$-stationary point we can now state our main result:

\begin{theorem}\label{thm:sgdsp}
	Assume that $f$ satisfies Assumption \ref{asm:smooth}, $\eta$ satisfies Assumption \ref{asm:eta}, and $L$ satisfies Assumption \ref{asm:kl}, i.e. $L(\theta) \le \mu \|\nabla L(\theta)\|^{1+\delta}$ for $L(\theta) \le \epsilon_{KL}$. Let  $\eta,B$ be chosen such that $\lambda := \frac{\eta \sigma^2}{B} = \tilde\Theta(\min(\epsilon^{2/\delta},\gamma^2))$, and let $T = \tilde\Theta(\eta^{-1}\lambda^{-1-\delta}) = \poly(\eta^{-1},\gamma^{-1})$. Assume that $\theta$ is initialized within $O(\sqrt{\lambda^{1+\delta}})$ of some $\theta^*$ satisfying $L(\theta^*) = O(\lambda^{1+\delta})$. Then for any $\zeta \in (0,1)$, with probability at least $1 - \zeta$, if $\{\theta_k\}$ follows \Cref{alg:sgdln} with parameters $\eta,\sigma,T$, there exists $k < T$ such that $\theta_k$ is an $(\epsilon,\gamma)$-stationary point of $\frac{1}{\lambda}\tilde{L}$.
\end{theorem}


Theorem \ref{thm:sgdsp} guarantees that Algorithm \ref{alg:sgdln} will hit an $(\epsilon,\gamma)$-stationary point of $\frac{1}{\lambda}\tilde{L}$ within a polynomial number of steps in $\epsilon^{-1},\gamma^{-1}$. In particular, when $\delta = \frac{1}{2}$, \Cref{thm:sgdsp} guarantees convergence within $\tilde O(\epsilon^{-6} + \gamma^{-3})$ steps. 
The condition that $\theta_0$ is close to an approximate global minimizer $\theta^*$ is not a strong assumption as recent methods have shown that overparameterized models can easily achieve zero training loss in the kernel regime (see \Cref{sec:NTK}).
However, in practice these minimizers of the training loss generalize poorly \citep{arora2019exact}. Theorem \ref{thm:sgdsp} shows that Algorithm \ref{alg:sgdln} can then converge to a stationary point of the regularized loss which has better generalization guarantees (see Section \ref{sec:generalization}). \Cref{thm:sgdsp} also generalizes the local analysis in \citet{blanc2019implicit} to a global result with weaker assumptions on the learning rate $\eta$. For a full comparison with \citet{blanc2019implicit}, see section \ref{sec:blanccomparison}.

\section{Proof Sketch}
\label{sec:sketch}
The proof of convergence to an $(\epsilon,\varphi)$-stationary point of $\frac{1}{\lambda} \tilde{L}$ has two components. In \Cref{sec:sketch:coupling}, we pick a reference point $\theta^*$ and analyze the behavior of \Cref{alg:sgdln} in a neighborhood of $\theta^*$. In \Cref{sec:sketch:convergence}, we repeat this local analysis with a sequence of reference points $\{\theta_m^*\}$.

\begin{figure}[t]
	\centering
	\includestandalone[width=0.8\textwidth]{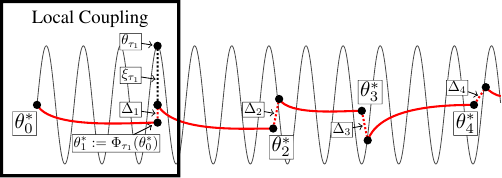}
	\caption{\textbf{Local Coupling: } The local coupling decomposes $\theta$ as $\theta_{\tau_1} = \Phi_{\tau_1}(\theta_0^*) + \xi_{\tau_1} + \Delta_1$. $\Phi_{\tau_1}(\theta_0^*)$ denotes $\tau_1$ steps of gradient descent on the regularized loss $\tilde{L}$ (denoted by the solid red curve), $\xi_{\tau_1}$ is a mean zero oscillating process (denoted by the dotted black line), and $\Delta_1$ is a small error term (denoted by the dotted red line). \textbf{Global Convergence:} By repeating this local coupling with a sequence of reference points $\{\theta_m^*\}_m$, we prove convergence to a stationary point of $\frac{1}{\lambda} \tilde L$.}
	\label{fig:localcouplingsketch}
\end{figure}


\subsection{Local Coupling}\label{sec:sketch:coupling}
Let $\Phi_k(\cdot)$ denote $k$ steps of gradient descent on the regularized loss $\tilde{L}$, i.e.
\begin{align}
	\Phi_0(\theta) = \theta \qqtext{and} \Phi_{k+1}(\theta) = \Phi_k(\theta) - \eta \nabla \tilde{L}(\Phi_k(\theta)), \label{eq:regularizedtrajectory}
\end{align}
where $\tilde{L}(\theta) = L(\theta) + \lambda R(\theta)$ is the    regularized loss defined in \Cref{eq:regdef}. \Cref{lem:sketch:coupling} states that if $\theta$ is initialized at an approximate global minimizer $\theta^*$ and follows \Cref{alg:sgdln}, there is a small mean zero random process $\xi$ such that $\theta_k \approx \Phi_k(\theta^*) + \xi_k$:

\begin{lemma}\label{lem:sketch:coupling}
	Let
	\begin{align*}
	\iota = c \log \frac{d}{\lambda \zeta}, && \mathscr{X} = \sqrt{\frac{2\lambda n d \iota}{\nu}}, && \mathscr{L} = c \lambda^{1+\delta}, && \mathscr{D} = c\sqrt{\mathscr{L}}\iota, && \mathscr{M} = \frac{\mathscr{D}}{\nu}, && \mathscr{T} = \frac{1}{c^2 \eta \mathscr{X}\iota},
	\end{align*}
	where $c$ is a sufficiently large constant.
	Assume $f$ satisfies \Cref{asm:smooth} and $\eta$ satisfies \Cref{asm:eta}. Let $\theta$ follow \Cref{alg:sgdln} starting at $\theta^*$ and assume that $L(\theta^*) \le \mathscr{L}$ for some $0 < \delta \le 1/2$. Then there exists a random process $\{\xi_k\}$ such that for any $\tau \le \mathscr{T}$ satisfying  $\max_{k \le \tau} \|\Phi_k(\theta^*) - \theta^*\| \le 8\mathscr{M}$, with probability at least $1-10 d \tau e^{-\iota}$ we have simultaneously for all $k \le \tau$,
	\begin{align*}
		\norm{\theta_k - \xi_k - \Phi_k(\theta^*)} \le \mathscr{D}, \qquad \E[\xi_k]=0, \qqtext{and} \|\xi_k\| \le \mathscr{X}.
	\end{align*}
\end{lemma}

Note that because $\mathscr{M} \ge \mathscr{D}$, the error term $\mathscr{D}$ is at least $8$ times smaller than the movement in the direction of the regularized trajectory $\Phi_\tau(\theta^*)$, which will allow us to prove convergence to an $(\epsilon,\gamma)$-stationary point of $\frac{1}{\lambda}\tilde L$ in \Cref{sec:sketch:convergence}.




Toward simplifying the update in \Cref{alg:sgdln}, we define $L^{(k)}$ to be the true loss without label noise on batch $\mathcal{B}^{(k)}$. The label-noise update $\hat{L}^{(k)}(\theta_k)$ is an unbiased perturbation of the mini-batch update: $\nabla \hat{L}^{(k)}(\theta_k) = \nabla L^{(k)}(\theta_k) - \frac{1}{B}\sum_{i\in \mathcal{B}^{(k)}}\epsilon_i^{(k)} \nabla f_i(\theta_k)$. We decompose the update rule into three parts:
\begin{align}
	\theta_{k+1} 
	&= \theta_k - \underbrace{\eta \nabla L(\theta_k)}_\text{gradient descent} - \underbrace{\eta[\nabla L^{(k)}(\theta_k) - \nabla L(\theta_k)]}_\text{minibatch noise} + \underbrace{\frac{\eta}{B} \sum_{i \in \mathcal{B}^{(k)}} \epsilon_i^{(k)} \nabla f_i(\theta_k)}_\text{label noise}. \label{eq:labelnoisedecomposition}
\end{align}
Let $m_k = -\eta[\nabla L^{(k)}(\theta_k) - \nabla L(\theta_k)]$ denote the minibatch noise.  Throughout the proof we will show that the minibatch noise is dominated by the label noise. We will also decompose the label noise into two terms. The first, $\epsilon_k^*$ will represent the label noise if the gradient were evaluated at $\theta^*$ whose distribution does not vary with $k$. The other term, $z_k$ represents the change in the noise due to evaluating the gradient at $\theta_k$ rather than $\theta^*$. More precisely, we have
\begin{align*}
	\epsilon_k^* = \frac{\eta}{B} \sum_{i \in \mathcal{B}^{(k)}} \epsilon_i^{(k)} \nabla f_i(\theta^*) \qqtext{and} z_k = \frac{\eta}{B} \sum_{i \in \mathcal{B}^{(k)}} \epsilon_i^{(k)} [\nabla f_i(\theta_k) - \nabla f_i(\theta^*)].
\end{align*}
We define
$
	G(\theta) = \frac{1}{n} \sum_i \nabla f_i(\theta) \nabla f_i(\theta)^T
$
to be the covariance of the model gradients. Note that $\epsilon_k^*$ has covariance $\eta \lambda G(\theta^*)$. To simplify notation in the Taylor expansions, we will use the following shorthand to refer to various quantities evaluated at $\theta^*$:

\begin{align*}
	G = G(\theta^*), && \nabla^2 L = \nabla^2 L(\theta^*), && \nabla^3 L = \nabla^3 L(\theta^*), && \nabla R = \nabla R(\theta^*).
\end{align*}

First we need the following standard decompositions of the Hessian:
\begin{proposition}\label{prop:hessiandecomp}
	For any $\theta \in \mathbb{R}^d$ we can decompose $\nabla^2 L(\theta) = G(\theta) + E(\theta)$ where $E(\theta) = \frac{1}{n} \sum_{i=1}^n (f_i(\theta)-y_i)\nabla^2 f_i(\theta)$ satisfies $\|E(\theta)\| \le \sqrt{2\rho_fL(\theta)}$ where $\rho_f$ is defined in \Cref{asm:smooth}.
\end{proposition}
The matrix $G$ in \Cref{prop:hessiandecomp} is known as the Gauss-Newton term of the Hessian. We can now Taylor expand \Cref{alg:sgdln} and \Cref{eq:regularizedtrajectory} to first order around $\theta^*$:
\begin{align*}
	\Phi_{k+1}(\theta^*) &\approx \Phi_{k}(\theta^*) - \eta \left[\nabla L + \nabla^2 L(\Phi_{k}(\theta^*)-\theta^*)\right], \\
	\theta_{k+1} &\approx \theta_k - \eta \left[\nabla L + \nabla^2 L(\theta_k - \theta^*)\right] + \epsilon_k^*.
\end{align*}
We define $v_k = \theta_k - \Phi_k(\theta^*)$ to be the deviation from the regularized trajectory. Then subtracting these two equations gives
\begin{align}
	v_{k+1}	&\approx (I - \eta \nabla^2 L)v_k + \epsilon_k^*  \approx (I - \eta G)v_k + \epsilon_k^*, \nonumber
\end{align}
\anote{I removed "+higher order terms" because it's implied by the $\approx$}
where we used \Cref{prop:hessiandecomp} to replace $\nabla^2 L$ with $G$. Temporarily ignoring the higher order terms, we define the random process $\xi$ by
\begin{align}
	\xi_{k+1} = (I - \eta G)\xi_k + \epsilon_k^* \qqtext{and} \xi_0 = 0. \label{eq:sketch:OUdef}
\end{align}
The process $\xi$ is referred to as an Ornstein Uhlenbeck process and it encodes the movement of $\theta$ to first order around $\theta^*$. We defer the proofs of the following properties of $\xi$ to \Cref{sec:appendix:missingproofs}:
\anote{@TengyuMa removed bullet 1 and made bullets 2/3 inline}
\begin{proposition}\label{prop:sketch:OUproperties}~
	For any $k \ge 0$, with probability at least $1-2de^{-\iota}$, $\|\xi_k\| \le \mathscr{X}$. In addition, as $k \to \infty$, $\E[\xi_k\xi_k^T] \to \lambda \Pi_G (2 - \eta G)^{-1}$ where $\Pi_G$ is the projection onto the span of $G$.
\end{proposition}
We can now analyze the effect of $\xi_k$ on the second order Taylor expansion. Let $r_k = \theta_k - \Phi_k(\theta^*) - \xi_k$ be the deviation of $\theta$ from the regularized trajectory after removing the Ornstein Uhlenbeck process $\xi$. \Cref{lem:sketch:coupling} is equivalent to $\Pr[\|r_\tau\| \ge \mathscr{D}] \le 10\tau d e^{-\iota}$.

We will prove by induction that $\|r_k\| \le \mathscr{D}$ for all $k \le t$ with probability at least $1-10 t d e^{-\iota}$ for all $t \le \tau$. The base case follows from $r_0 = 0$ so assume the result for some $t \ge 0$. The remainder of this section will be conditioned on the event $\|r_k\| \le \mathscr{D}$ for all $k \le t$. $O(\cdot)$ notation will only be used to hide absolute constants that do not change with $t$ and will additionally not hide dependence on the absolute constant $c$. The following proposition fills in the missing second order terms in the Taylor expansion around $\theta^*$ of $r_k$:

\begin{proposition}\label{prop:sketch:taylor2}
With probability at least $1-2de^{-\iota}$,
	\begin{align*}
	r_{k+1} 
	&= (I - \eta G)r_k - \eta\left[\frac{1}{2}\nabla^3 L(\xi_k,\xi_k) - \lambda \nabla R \right] + m_k + z_k + \tilde O\left(c^{5/2}\eta \lambda^{1+\delta}\right) 
	\end{align*}
\end{proposition}

The intuition for the implicit regularizer $R(\theta)$ is that by \Cref{prop:sketch:OUproperties,prop:hessiandecomp},
\begin{align*}
	\E[\xi_k\xi_k^T] \to \Pi_G\lambda(2-\eta G)^{-1} \approx \lambda(2-\eta \nabla^2 L)^{-1}.
\end{align*}
Therefore, when averaged over long timescales,
\begin{align*}
	\frac{1}{2}\E[\nabla^3 L(\xi_k,\xi_k)]
	&\approx \frac{\lambda}{2} \nabla^3 L\left[(2-\eta \nabla^2 L)^{-1}\right] \\
	&= \eval{\lambda \nabla \left[-\frac{1}{2 \eta}\tr \log \left(1-\frac{\eta}{2} \nabla^2 L(\theta)\right)\right]}_{\theta = \theta^*} \\
	&= \lambda \nabla R.
\end{align*}
The second equality follows from the more general equality that for any matrix function $A$ and any scalar function $h$ that acts independently on each eigenvalue, $\nabla (\tr h(A(\theta))) = (\nabla A(\theta))(h'(A(\theta)))$ which follows from the chain rule. The above equality is the special case when $A(\theta) = \nabla^2 L(\theta)$ and $h(x) = -\frac{1}{\eta} \log\left(1 - \frac{\eta}{2} x\right)$, which satisfies $h'(x) = \frac{1}{2-\eta x}$.

The remaining details involve concentrating the mean zero error terms $m_k,z_k$ and showing that $\E[\xi_k\xi_k^T]$ \textit{does} concentrate in the directions with large eigenvalues and that the directions with small eigenvalues, in which the covariance does not concentrate, do not contribute much to the error. This yields the following bound:
\begin{proposition}\label{prop:sketch:OUcov} With probability at least $1-10d e^{-\iota}$, $\|r_{t+1}\| = \tilde O\left(\frac{\lambda^{1/2+\delta/2}}{\sqrt{c}}\right)$.
\end{proposition}

The proof of \Cref{prop:sketch:OUcov} can be found in \Cref{sec:appendix:missingproofs}. Finally, because $\mathscr{D} = \tilde O(c^{5/2} \lambda^{1/2+\delta/2})$, $\|r_{t+1}\| \le \mathscr{D}$ for sufficiently large $c$. This completes the induction and the proof of \Cref{lem:sketch:coupling}.

\paragraph{Comparison with \citet{blanc2019implicit}}\label{sec:blanccomparison}
Like \citet{blanc2019implicit}, \Cref{lem:sketch:coupling} shows that $\theta$ locally follows the trajectory of gradient descent on an implicit regularizer $R(\theta)$. However, there are a few crucial differences:
\begin{itemize}[leftmargin=*]
	\item Because we do not assume we start near a global minimizer where $L = 0$, we couple to a regularized loss $\tilde{L} =  L + \lambda R$ rather than just the regularizer $R(\theta)$. In this setting there is an additional correction term to the Hessian (\Cref{prop:hessiandecomp}) that requires carefully controlling the value of the loss across reference points to prove convergence to a stationary point.
	\item The analysis in \citet{blanc2019implicit} requires $\eta,\tau$ to be chosen in terms of the condition number of $\nabla^2 L$ which can quickly grow during training as $\nabla^2 L$ is changing. This makes it impossible to directly repeat the argument. We avoid this by precisely analyzing the error incurred by small eigenvalues, allowing us to prove convergence to an $(\epsilon,\gamma)$ stationary point of $\frac{1}{\lambda} \tilde{L}$ for fixed $\eta,\lambda$ even if the smallest nonzero eigenvalue of $\nabla^2 L$ converges to $0$ during training.
	\item Unlike in \citet{blanc2019implicit}, we do not require the learning rate $\eta$ to be small. Instead, we only require that $\lambda$ scales with $\epsilon$ which can be accomplished either by decreasing the learning rate $\eta$ or increasing the batch size $B$. This allows for stronger implicit regularization in the setting when $\eta$ is large (see \Cref{sec:strengthlargeeta}). In particular, our regularizer $R(\theta)$ changes with $\eta$ and is only equal to the regularizer in \citet{blanc2019implicit} in the limit $\eta \to 0$.
	
\end{itemize}

\subsection{Global Convergence}\label{sec:sketch:convergence}
In order to prove convergence to an $(\epsilon,\gamma)$-stationary point of $\frac{1}{\eta} \nabla \tilde{L}$, we will define a sequence of reference points $\theta_m^*$ and coupling times $\{\tau_m\}$ and repeatedly use a version of \Cref{lem:sketch:coupling} to describe the long term behavior of $\theta$. For notational simplicity, given a sequence of coupling times $\{\tau_m\}$, define $T_m = \sum_{k < m} \tau_k$ to be the total number of steps until we have reached the reference point $\theta_m^*$.

To be able to repeat the local analysis in \Cref{lem:sketch:coupling} with multiple reference points, we need a more general coupling lemma that allows the random process $\xi$ defined in each coupling to continue where the random process in the previous coupling ended. To accomplish this, we define $\xi$ outside the scope of the local coupling lemma:
\begin{definition}\label{def:globalOU}
	Given a sequence of reference points $\{\theta_m^*\}$ and a sequence of coupling times $\{\tau_m\}$, we define the random process $\xi$ by $\xi_0 = 0$, and for $k \in [T_m,T_{m+1})$,
	\begin{align*}
		\epsilon_k^* = \frac{\eta}{B}\sum_{i \in \mathcal{B}^{(k)}} \epsilon^{(k)}_i \nabla f_i(\theta_m^*) \qqtext{and} \xi_{k+1} = (I - \eta G(\theta_m^*))\xi_k + \epsilon_k^*.
	\end{align*}
\end{definition}
Then we can prove the following more general coupling lemma:

\begin{lemma}\label{lem:convergence:coupling}
	\sloppy
	Let $\mathscr{X}, \mathscr{L}, \mathscr{D}, \mathscr{M}, \mathscr{T}$ be defined as in \Cref{lem:sketch:coupling}. Assume $f$ satisfies \Cref{asm:smooth} and $\eta$ satisfies \Cref{asm:eta}. Let $\Delta_m = \theta_{T_m} - \xi_{T_m} - \theta_m^*$ and assume that $\|\Delta_m\| \le \mathscr{D}$ and $L(\theta_m^*) \le \mathscr{L}$ for some $0 < \delta \le 1/2$. Then for any $\tau_m \le \mathscr{T}$ satisfying $\max_{k \in [T_m,T_{m+1})} \|\Phi_{k-T_m}(\theta_m^* + \Delta_m) - \theta_m^*\| \le 8\mathscr{M}$, with probability at least $1-10 d \tau_m e^{-\iota}$ we have simultaneously for all $k \in (T_m,T_{m+1}]$,
	\begin{align*}
		\norm{\theta_k - \xi_k - \Phi_{k-T_m}(\theta_m^* + \Delta_m)} \le \mathscr{D}, \qquad \E[\xi_k]=0, \qqtext{and} \|\xi_k\| \le \mathscr{X}.
	\end{align*}
\end{lemma}
Unlike in \Cref{lem:sketch:coupling}, we couple to the regularized trajectory starting at $\theta_m^* + \Delta_m$ rather than at $\theta_m^*$ to avoid accumulating errors (see \Cref{fig:localcouplingsketch}). The proof is otherwise identical to that of \Cref{lem:sketch:coupling}. 

The proof of \Cref{thm:sgdsp} easily follows from the following lemma which states that we decrease the regularized loss $\tilde L$ by at least $\mathscr{F}$ after every coupling:
\begin{lemma}\label{lem:convergence:1step}
	Let $\mathscr{F} = \frac{\mathscr{D}^2}{\eta \nu \mathscr{T}}$. Let $\Delta_m = \theta_{T_m} - \xi_{T_m} - \theta_m^*$ and assume $\|\Delta_m\| \le \mathscr{D}$ and $L(\theta_m^*) \le \mathscr{L}$. Then if $\theta_{T_m}$ is not an $(\epsilon,\gamma)$-stationary point, there exists some $\tau_m < \mathscr{T}$ such that if we define
	\begin{align*}
		\theta_{m+1}^* = \Phi_{\tau_n}(\theta_m^* + \Delta_m) \qqtext{and} \Delta_{m+1} = \theta_{T_{m+1}} - \xi_{T_{m+1}} - \theta_{m+1}^*,
	\end{align*}
	then with probability $1-10d\tau_m e^{-\iota}$,
	\begin{align*}
		\tilde L(\theta_{m+1}^*) \le L(\theta_m^*) - \mathscr{F}, \qquad \|\Delta_{m+1}\| \le \mathscr{D} \qqtext{and} L(\theta_{m+1}^*) \le \mathscr{L}.
	\end{align*}
\end{lemma}

We defer the proofs of \Cref{lem:convergence:coupling} and \Cref{lem:convergence:1step} to \Cref{sec:appendix:missingproofs}. \Cref{thm:sgdsp} now follows directly from repeated applications of \Cref{lem:convergence:1step}:
\begin{proof}[Proof of \Cref{thm:sgdsp}]
	By assumption there exists some $\theta_0^*$ such that $L(\theta_0^*) \le \mathscr{L}$ and $\|\theta_0 - \theta_0^*\| \le \mathscr{D}$. Then so long as $\theta_{T_m}$ is not an $(\epsilon,\gamma)$-stationary point, we can inductively apply \Cref{lem:convergence:1step} to get the existence of coupling times $\{\tau_m\}$ and reference points $\{\theta_m^*\}$ such that for any $m \ge 0$, with probability $1-10 d T_m e^{-\iota}$ we have $\tilde L(\theta_{m}^*) \le \tilde L(\theta_0^*) - m \mathscr{F}$. As $\tilde L(\theta_0^*) - \tilde L(\theta_m^*) = O(\lambda)$, this can happen for at most $m = O\left(\frac{\lambda}{\mathscr{F}}\right)$ reference points, so at most $T = O\left(\frac{\lambda\mathscr{T}}{\mathscr{F}}\right) = \tilde O\left(\eta^{-1}\lambda^{-1-\delta}\right)$ iterations of \Cref{alg:sgdln}. By the choice of $\iota$, this happens with probability $1-10 d T e^{-\iota} \ge 1-\zeta$.
\end{proof}

\section{Experiments}\label{sec:experiments}
\begin{figure}[t]
	\centering
	\begin{subfigure}[c]{0.78\textwidth}
	\includegraphics[width=\linewidth]{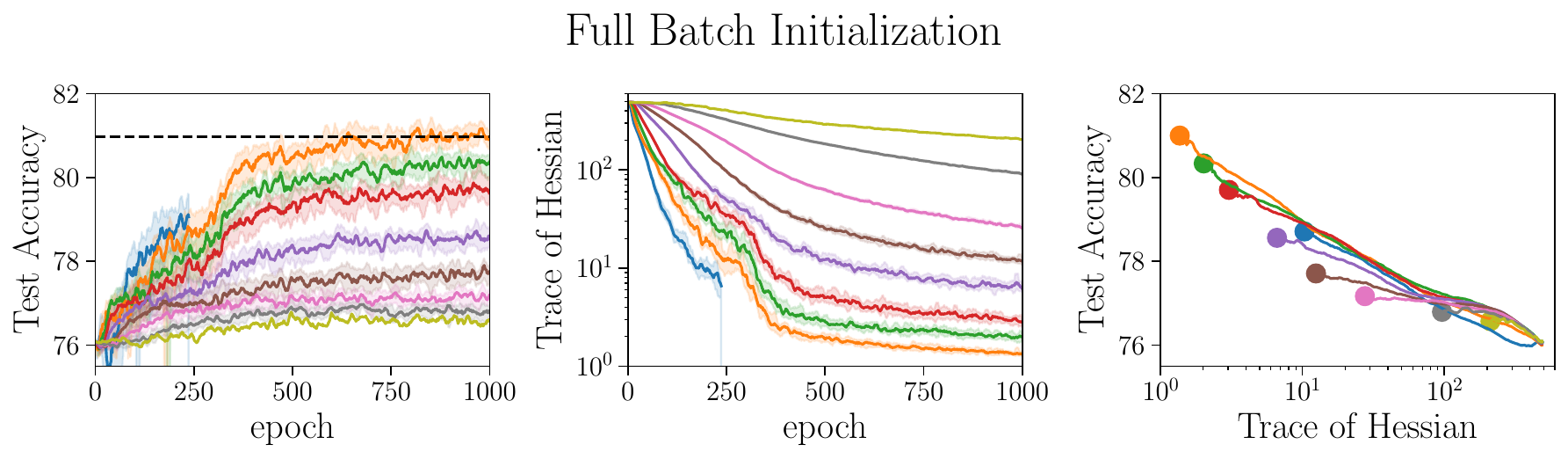}\\
	\includegraphics[width=\linewidth]{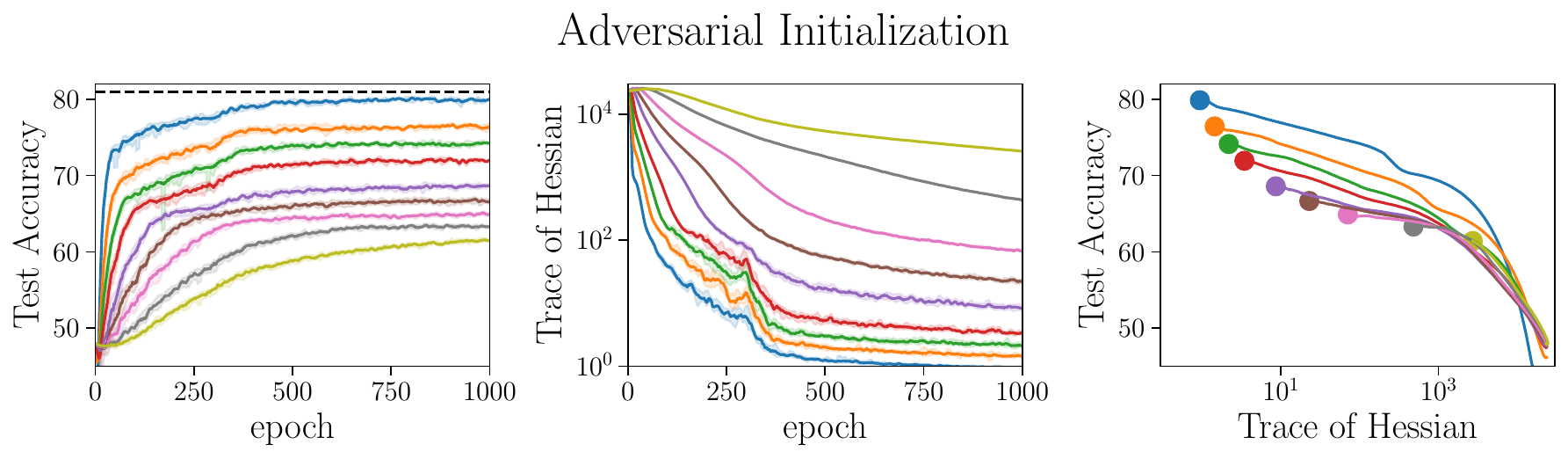}
	\end{subfigure}
	\begin{subfigure}[c]{0.2\textwidth}
	\includegraphics[width=\linewidth]{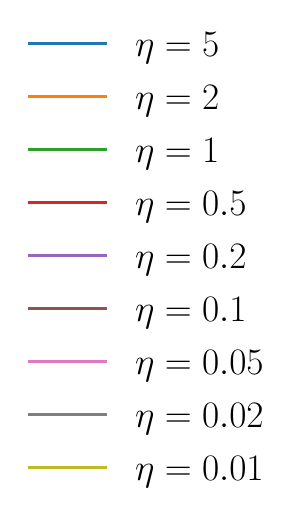}
	\end{subfigure}
	\caption{\textbf{Label Noise SGD escapes poor global minimizers.} The left column displays the training accuracy over time, the middle column displays the value of $\tr \nabla^2 L(\theta)$ over time which we use to approximate the implicit regularizer $R(\theta)$, and the right column displays their correlation. The horizontal dashed line represents the minibatch SGD baseline with random initialization. We report the median results over $3$ random seeds and shaded error bars denote the min/max over the three runs. The correlation plot uses a running average of $100$ epochs for visual clarity.}
	\label{fig:experiments}
\end{figure}

In order to test the ability of SGD with label noise to escape poor global minimizers and converge to better minimizers, we initialize \Cref{alg:sgdln} at global minimizers of the training loss which achieve $100\%$ training accuracy yet generalize poorly to the test set. Minibatch SGD would remain fixed at these initializations because both the gradient and the noise in minibatch SGD vanish at any global minimizer of the training loss. We show that SGD with label noise escapes these poor initializations and converges to flatter minimizers that generalize well, which supports \Cref{thm:sgdsp}. We run experiments with two initializations:

\textbf{Full Batch Initialization:} We run full batch gradient descent with random initialization until convergence to a global minimizer. We call this minimizer the full batch initialization. The final test accuracy of the full batch initialization was 76\%.

\textbf{Adversarial Initialization:} Following  \citet{liu2019bad}, we generate an adversarial initialization with final test accuracy $48\%$ that achieves zero training loss by first teaching the network to memorize random labels and then training it on the true labels. See \Cref{sec:appendix:experiments} for full details.


Experiments were run with ResNet18 on CIFAR10 \citep{Krizhevsky09learningmultiple} without data augmentation or weight decay. The experiments were conducted with randomized label flipping with probability $0.2$ (see \Cref{sec:classification} for the extension of \Cref{thm:sgdsp} to classification with label flipping), cross entropy loss, and batch size 256. Because of the difficulty in computing the regularizer $R(\theta)$, we approximate it by its lower bound $\tr \nabla^2 L(\theta)$. \Cref{fig:experiments} shows the test accuracy and $\tr \nabla^2 L$ throughout training.

SGD with label noise escapes both zero training loss initializations and converges to flatter minimizers that generalize much better, reaching the SGD baseline from the fullbatch initialization and getting within $1\%$ of the baseline from the adversarial initialization. The test accuracy in both cases is strongly correlated with $\tr \nabla^2 L$. The strength of the regularization is also strongly correlated with $\eta$, which supports \Cref{thm:sgdsp}. See \Cref{fig:experiments_m} for experimental results for SGD with momentum.

\section{Extensions}\label{sec:extensions}

\subsection{Classification}\label{sec:classification}
We restrict $y_i \in \{-1,1\}$, let $l:\mathbb{R} \to \mathbb{R}^+$ be an arbitrary loss function, and $p \in (0,1)$ be a smoothing factor. Examples of $l$ include logistic loss, exponential loss, and square loss (see \Cref{table:differentlosses}). We define $\bar l$ to be the expected smoothed loss where we flip each label with probability $p$:
\begin{align}
\bar{l}(x) = p l(-x) + (1-p) l(x).
\end{align}
We make the following mild assumption on the smoothed loss $\bar{l}$ which is explicitly verified for the logistic loss, exponential loss, and square loss in \Cref{sec:quadraticapprox}:

\begin{assumption}[Quadratic Approximation]\label{asm:quadraticapprox}
	If $c \in \mathbb{R}$ is the unique global minimizer of $\bar{l}$, there exist constants $\epsilon_{Q}>0,\nu > 0$ such that if $\bar{l}(x) \le \epsilon_Q$ then,
	\begin{align}
		(x-c)^2 \le \nu (\bar{l}(x)-\bar{l}(c)).
	\end{align}
	In addition, we assume that $\bar l', \bar l''$ are $\rho_l$, $\kappa_l$ Lipschitz respectively restricted to the set $\{x : \bar{l}(x) \le \epsilon_Q\}$.
\end{assumption}

Then we define the per-sample loss and the sample loss as:
\begin{align}
\ell_i(\theta) = \bar{l}(y_i f_i(\theta)) - \bar{l}(c) \qqtext{and} L(\theta) = \frac{1}{n} \sum_{i=1}^n \ell_i(\theta).
\end{align}

We will follow \Cref{alg:sgdls}:

\begin{algorithm}[H]
\SetAlgoLined
\KwIn{$\theta_0$, step size $\eta$, smoothing constant $p$, batch size $B$, steps $T$, loss function $l$}
\For{$k = 0$ to $T-1$}
{
Sample batch $\mathcal{B}^{(k)} \sim [n]^{B}$ uniformly and sample $\sigma^{(k)}_i = 1,-1$ with probability $1-p,p$ respectively for $i \in \mathcal{B}^{(k)}$.\\
Let $\hat{\ell}_i^{(k)}(\theta) = l[\sigma^{(k)}_i y_{i} f_{i}(\theta)]$ and $\hat L^{(k)} = \frac{1}{B} \sum_{i \in \mathcal{B}^{(k)}} \hat{\ell}_i^{(k)}$. \\
$\theta_{k+1} \leftarrow \theta_k - \eta \nabla \hat L^{(k)}(\theta_k)$
}
\caption{SGD with Label Smoothing}
\label{alg:sgdls}
\end{algorithm}
Now note that the noise per sample from label smoothing at a zero loss global minimizer $\theta^*$ can be written as
\begin{align}
	\nabla \hat{\ell}^{(k)}_i(\theta^*) - \nabla \ell_{i}(\theta^*) = \epsilon \nabla f_i(\theta^*)
\end{align}
where
\begin{align}
\epsilon = \begin{cases}
 	p(l'(c) + l'(-c)) \text{with probability } 1-p \\
 	-(1-p)(l'(c) + l'(-c)) \text{with probability } p
 \end{cases}	
\end{align}
so $E[\epsilon] = 0$ and
\begin{align}
\sigma^2 = E[\epsilon^2] = 	p(1-p)(l'(c) + l'(-c))^2,
\end{align}
which will determine the strength of the regularization in \Cref{thm:sgdlssp}. Finally, in order to study the local behavior around $c$ we define $\alpha = \bar{l}''(c) > 0$ by \Cref{asm:quadraticapprox}. Corresponding values for $c,\sigma^2,\alpha$ for logistic loss, exponential loss, and square loss are given in \Cref{table:differentlosses}.
\begin{table}[t]
\center
\tabulinesep=1.2mm
\begin{tabu}{|c|c|c|c|c|}
\hline
& $l(x)$ & $c = \arg\min_x \bar l(x)$ & $\sigma^2 = E[\epsilon^2]$ & $\alpha = \bar l''(c)$ \\ \hline
Logistic Loss & $\log \left[1+e^{-x}\right]$ & $\log \frac{1-p}{p}$ & $p(1-p)$ & $p(1-p)$                    \\ \hline
Exponential Loss & $e^{-x}$ & $\frac{1}{2}\log \frac{1-p}{p}$ & $1$ & $2\sqrt{p(1-p)}$ \\ \hline
Square Loss & $\frac{1}{2}(x-1)^2$ & $1-2p$ & $4p(1-p)$ & $1$ \\ \hline
\end{tabu}
\caption{Values of $l(x)$, $c$, $\sigma^2$, $\alpha$ for different binary classification loss functions}
\label{table:differentlosses}
\end{table}

As before we define:
\begin{align}
	R(\theta) = -\frac{1}{2 \eta \alpha}\tr \log \left(1-\frac{\eta}{2} \nabla^2 L(\theta)\right), && \lambda = \frac{\eta \sigma^2}{B}, && \tilde{L}(\theta) = L(\theta) + \lambda R(\theta). \label{eq:ls:regdef}
\end{align}
Our main result is a version of \Cref{thm:sgdsp}:

\begin{theorem}\label{thm:sgdlssp}
	Assume that $f$ satisfies Assumption \ref{asm:smooth}, $\eta$ satisfies Assumption \ref{asm:eta}, $L$ satisfies Assumption \ref{asm:kl} and $l$ satisfies \Cref{asm:quadraticapprox}.
	Let  $\eta, B$ be chosen such that $\lambda := \frac{\eta \sigma^2}{B} = \tilde\Theta(\min(\epsilon^{2/\delta},\gamma^2))$, and let $T = \tilde\Theta(\eta^{-1}\lambda^{-1-\delta}) = \poly(\eta^{-1},\gamma^{-1})$. Assume that $\theta$ is initialized within $O(\sqrt{\lambda^{1+\delta}})$ of some $\theta^*$ satisfying $L(\theta^*) = O(\lambda^{1+\delta})$. Then for any $\zeta \in (0,1)$, with probability at least $1 - \zeta$, if $\{\theta_k\}$ follows \Cref{alg:sgdls} with parameters $\eta,\sigma,T$, there exists $k < T$ such that $\theta_k$ is an $(\epsilon,\gamma)$-stationary point of $\frac{1}{\lambda}\tilde{L}$.
\end{theorem}

\subsection{SGD with Momentum}\label{sec:momentum}
\begin{figure}[t]
	\centering
	\begin{subfigure}[c]{0.78\textwidth}
	\includegraphics[width=\linewidth]{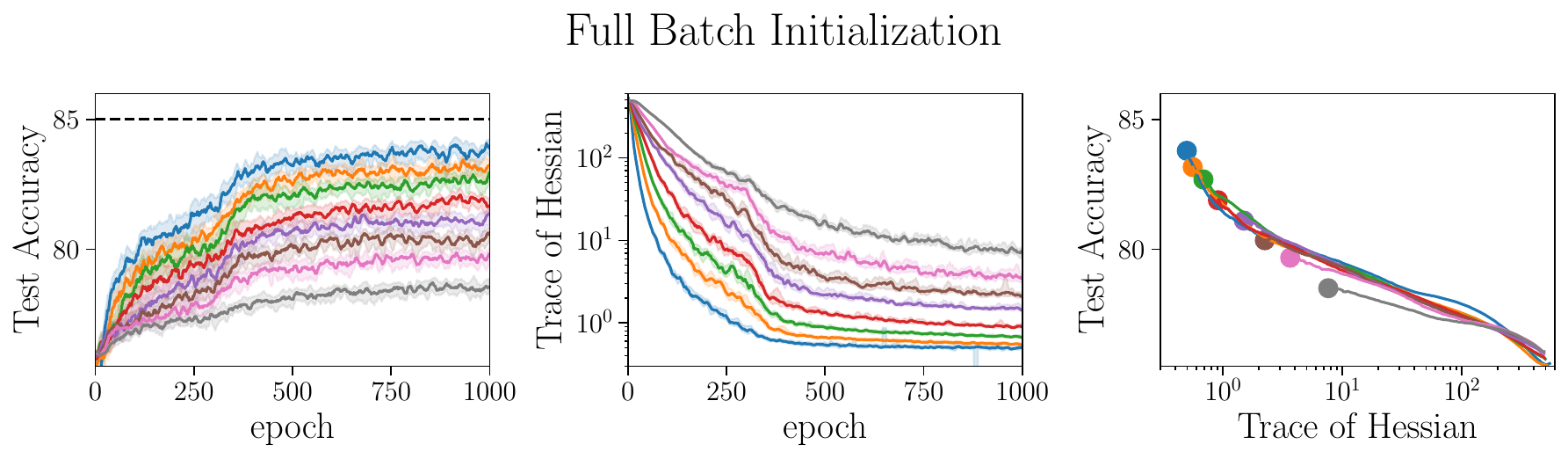}\\
	\includegraphics[width=\linewidth]{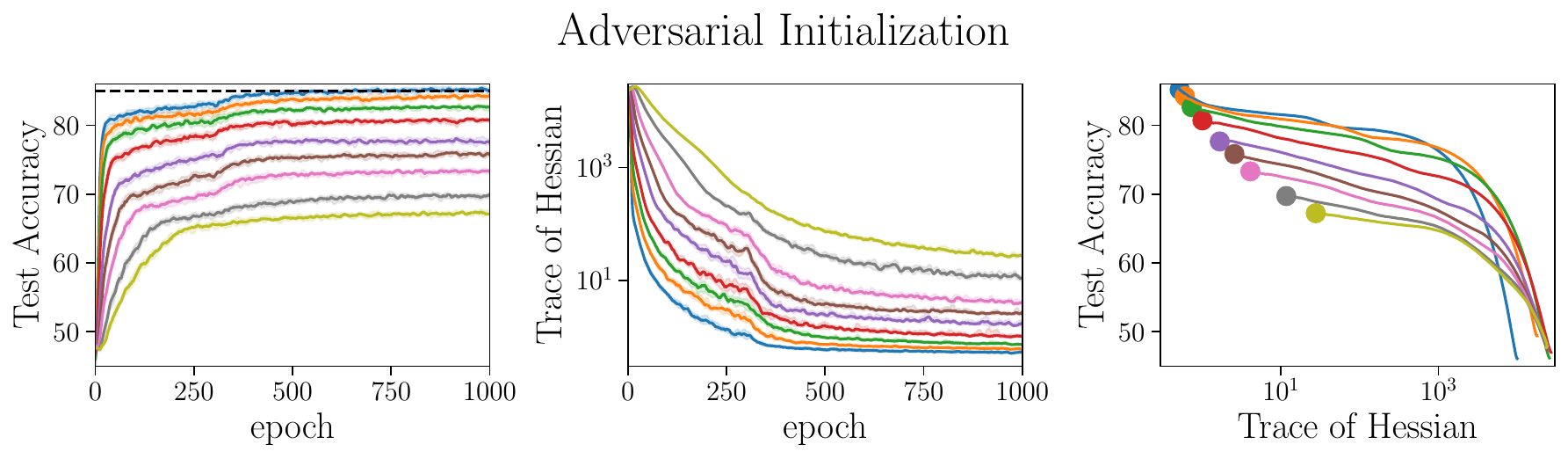}
	\end{subfigure}
	\begin{subfigure}[c]{0.2\textwidth}
	\includegraphics[width=\linewidth]{Figures/experiments/legend.pdf}
	\end{subfigure}
	\caption{\textbf{Label Noise SGD with Momentum ($\beta = 0.9$)} The left column displays the training accuracy over time, the middle column displays the value of $\tr \nabla^2 L(\theta)$ over time which we use to approximate the implicit regularizer $R(\theta)$, and the right column displays their correlation. The horizontal dashed line represents the minibatch SGD baseline with random initialization. We report the median results over $3$ random seeds and shaded error bars denote the min/max over the three runs. The correlation plot uses a running average of $100$ epochs for visual clarity.}
	\label{fig:experiments_m}
\end{figure}

We consider heavy ball momentum with momentum $\beta$, i.e. we replace the update in \Cref{alg:sgdln} with
\begin{align}
	\theta_{k+1} = \theta_k - \eta \nabla \hat L^{(k)}(\theta_k) + \beta(\theta_k - \theta_{k-1}).
\end{align}
We define:
\begin{align}
	R(\theta) = \frac{1+\beta}{2\eta} \tr \log\left(1 - \frac{\eta}{2(1+\beta)} \nabla^2 L(\theta)\right), && \lambda = \frac{\eta \sigma^2}{B(1-\beta)},
\end{align}
and as before $\tilde L(\theta) = L(\theta) + \lambda R(\theta)$. Let
\begin{align}
	\Phi_0(\theta) = \theta, && \Phi_{k+1}(\theta) = \Phi_k(\theta) - \eta \nabla \tilde L(\Phi_k(\theta)) + \beta(\Phi_k(\theta) - \Phi_{k-1}(\theta))
\end{align}
represent gradient descent with momentum on $\tilde L$. Then we have the following local coupling lemma:

\begin{lemma}\label{lem:momentumcoupling}
	Let
	\begin{align}
	\mathscr{X} = \sqrt{\frac{2\lambda n^2 \iota}{\nu}}, && \mathscr{L} = c \lambda^{1+\delta}, && \mathscr{D} = c\sqrt{\mathscr{L}}\iota, && \mathscr{T} = \frac{1}{c^2 \eta \mathscr{X}\iota},
	\end{align}
	where $c$ is a sufficiently large constant.
	Assume $f$ satisfies \Cref{asm:smooth} and $\eta \le \frac{(2-\nu)(1+\beta)}{\ell}$. Let $\theta$ follow \Cref{alg:sgdln} with momentum parameter $\beta$ starting at $\theta^*$ and assume that $L(\theta^*) \le \mathscr{L}$ for some $0 < \delta \le 1/2$. Then there exists a random process $\{\xi_k\}$ such that for any $\tau \le \mathscr{T}$ satisfying  $\max_{k \le \tau} \|\Phi_k(\theta^*) - \theta^*\| \le 8\mathscr{D}$, with probability at least $1-10 d \tau e^{-\iota}$ we have simultaneously for all $k \le \tau$,
	\begin{align}
		\norm{\theta_k - \xi_k - \Phi_k(\theta^*)} \le \mathscr{D}, \qquad \E[\xi_k]=0, \qqtext{and} \|\xi_k\| \le \mathscr{X}.
	\end{align}
\end{lemma}

As in \Cref{lem:sketch:coupling}, the error is $8$ times smaller than the maximum movement of the regularized trajectory. Note that momentum increases the regularization parameter $\lambda$ by $\frac{1}{1-\beta}$. For the commonly used momentum parameter $\beta = 0.9$, this represents a $10\times$ increase in regularization, which is likely the cause of the improved performance in \Cref{fig:experiments_m} ($\beta = 0.9$) over \Cref{fig:experiments} ($\beta = 0$).

\subsection{Arbitrary Noise Covariances}\label{sec:arbitrarynoise}
The analysis in \Cref{sec:sketch:coupling} is not specific to label noise SGD and can be carried out for arbitrary noise schemes. Let $\theta$ follow $\theta_{k+1} = \theta_k - \eta \nabla L(\theta_k) + \epsilon_k$ starting at $\theta_0$ where $\epsilon_k \sim N(0,\eta \lambda \Sigma(\theta_k))$ and $\Sigma^{1/2}$ is Lipschitz. Given a matrix $S$ we define the regularizer $R_S(\theta) = \left\langle  S, \nabla^2 L(\theta) \right\rangle$. The matrix $S$ controls the weight of each eigenvalue. As before we can define $\tilde L_S(\theta) = L(\theta) + \lambda R_S(\theta)$ and $\Phi_{k+1}^S(\theta) = \Phi_k^S(\theta) - \eta \nabla \tilde L_S(\Phi_k(\theta))$ to be the regularized loss and the regularized trajectory respectively. Then we have the following version of \Cref{lem:sketch:coupling}:
\begin{proposition}\label{prop:arbitrarynoise}
Let $\theta$ be initialized at a minimizer $\theta^*$ of $L$. Assume $\nabla^2 L$ is Lipschitz, let $H = \nabla^2 L(\theta^*)$ and assume that $\Sigma(\theta^*) \preceq C H$ for some absolute constant $C$. Let $\mathscr{X} = \sqrt{\frac{C d \lambda \iota}{\nu}}$, $\mathscr{D} = c \lambda^{3/4} \iota$, and $\mathscr{T} = \frac{1}{c^2 \eta \mathscr{X} \iota}$ for a sufficiently large constant $c$. Then there exists a mean zero random process $\xi$ such that for any $\tau \le \mathscr{T}$ satisfying $\max_{k < \tau} \|\Phi_k(\theta^*) - \theta^*\| \le 8\mathscr{D}$ and with probability $1-10d \tau e^{-\iota}$, we have simultaneously for all $k \le \tau$:
\begin{align*}
	\|\theta_k - \xi_k - \Phi^S_k(\theta_0)\| \le \mathscr{D} \qqtext{and} \|\xi_k\| \le \mathscr{X},
\end{align*}
where $S$ is the unique fixed point of $S \leftarrow (I - \eta H)S(I - \eta H) + \eta \lambda \Sigma(\theta^*)$ restricted to $\mathrm{span}(H)$.
\end{proposition}
As in \Cref{lem:sketch:coupling}, the error is $8$ times smaller than the maximum movement of the regularized trajectory. Although \Cref{prop:arbitrarynoise} couples to gradient descent on $R_S$, $S$ is defined in terms of the Hessian and the noise covariance at $\theta^*$ and therefore depends on the choice of reference point. Because $R_S$ is changing, we cannot repeat \Cref{prop:arbitrarynoise} as in \Cref{sec:sketch:convergence} to prove convergence to a stationary point because there is no fixed potential. Although it is sometimes possible to relate $R_S$ to a fixed potential $R$, we show in \Cref{sec:cycling} that this is not generally possible by providing an example where minibatch SGD perpetually cycles. Exploring the properties of these continuously changing potentials and their connections to generalization is an interesting avenue for future work.

\section{Discussion}\label{sec:discussion}
\subsection{Sharpness and the Effect of Large Learning Rates}\label{sec:strengthlargeeta}
Various factors can control the strength of the implicit regularization in \Cref{thm:sgdsp}. Most important is the implicit regularization parameter $\lambda = \frac{\eta \sigma^2}{|B|}$. This supports the hypothesis that large learning rates and small batch sizes are necessary for implicit regularization \citep{goyal2017accurate,smith2017don}, and agrees with the standard linear scaling rule which proposes that for constant regularization strength, the learning rate $\eta$ needs to be inversely proportional to the batch size $|B|$. 

However, our analysis also uncovers an \textit{additional} regularization effect of large learning rates. Unlike the regularizer in \citet{blanc2019implicit}, the implicit regularizer $R(\theta)$ defined in \Cref{eq:regdef} is dependent on $\eta$. It is not possible to directly analyze the behavior of $R(\theta)$ as $\eta \to 2/\lambda_1$ where $\lambda_1$ is the largest eigenvalue of $\nabla^2 L$, as in this regime $R(\theta) \to \infty$ (see \Cref{fig:regplot}). If we let $\eta = \frac{2-\nu}{\lambda_1}$, then we can better understand the behavior of $R(\theta)$ by normalizing it by $\log 2/\nu$. This gives\footnote{Here we assume $\lambda_1 > \lambda_2$. If instead $\lambda_1 = \ldots = \lambda_k > \lambda_{k+1}$, this limit will be $k\|\nabla^2 L(\theta)\|_2$.}
\begin{align*}
	\frac{R(\theta)}{\log 2/\nu} = \sum_i \frac{R(\lambda_i)}{\log 2/\nu} = \|\nabla^2 L(\theta)\|_2 + O\left(\frac{1}{\log 2/\nu}\right) \xrightarrow{\nu \to 0} \|\nabla^2 L(\theta)\|_2
\end{align*}
so after normalization, $R(\theta)$ becomes a better and better approximation of the spectral norm $\|\nabla^2 L(\theta)\|$ as $\eta \to 2/\lambda_1$. $R(\theta)$ can therefore be seen as interpolating between $\tr \nabla^2 L(\theta)$, when $\eta \approx 0$, and $\|\nabla^2 L(\theta)\|_2$ when $\eta \approx 2/\lambda_1$. This also suggests that SGD with large learning rates may be more resilient to the edge of stability phenomenon observed in \citet{cohen2021gradient} as the implicit regularization works harder to control eigenvalues approaching $2/\eta$.

The sharpness-aware algorithm (SAM) of \cite{foret2020sharpness} is also closely related to $R(\theta)$. SAM proposes to minimize $\max_{\|\delta\|_2 \le \epsilon} L(\theta +\delta)$. At a global minimizer of the training loss,
\begin{align*}
	\max_{\|\delta\|_2 \le \epsilon} L(\theta^\ast+\delta) &= \max_{\|\delta\|_2 \le \epsilon}\frac12 \delta^\top \nabla^2 L(\theta^\ast) \delta + O(\epsilon^3) \approx \frac{\epsilon^2}{2} \|\nabla^2 L(\theta^\ast)\|_{2}.
	\end{align*}
The SAM algorithm is therefore explicitly regularizing the spectral norm of $\nabla^2 L(\theta)$, which is closely connected to the large learning rate regularization effect of $R(\theta)$ when $\eta \approx 2/\lambda_1$.
%
%

\subsection{Generalization Bounds}\label{sec:generalization}

\sloppy The implicit regularizer $R(\theta)$ is intimately connected to data-dependent generalization bounds, which measure the Lipschitzness of the network via the network Jacobian. Specifically, \citet{wei2019improved} propose the all-layer margin, which bounds the  $\textup{generalization error}\lesssim \frac{\sum_{l=1}^L \mathcal{C}_l}{\sqrt{n}} \sqrt{\frac1n \sum_{i=1}^n \frac{1}{m_F (x_i,y_i)^2}}$, where $\mathcal{C}_l$ depends only on the norm of the parameters and $m_F$ is the all-layer margin. The norm of the parameters is generally controlled by weight decay regularization, so we focus our discussion on the all-layer margin. Ignoring higher-order secondary terms, \citet[Heuristic derivation of Lemma 3.1]{wei2019improved} showed for a feed-forward network $f(\theta;x) = \theta_L \sigma(\theta_{L-1} \ldots \sigma(\theta_1x))$, the all-layer margin satisfies\footnote{The output margin is defined as $\min_i f_i (\theta) y_i$. The following uses Equation (3.3) and the first-order approximation provided \citet{wei2019improved} and the chain rule $\frac{\partial f}{\partial \theta_l}=\frac{\partial f }{\partial h_l} \frac{\partial h_l}{\partial \theta_{l-1} }= \frac{\partial f}{\partial h_l}h_{l-1}^\top$.}:
\begin{align*}
	\frac{1}{m_F(x,y)}\lesssim\frac{\|\{\frac{\partial f}{\partial \theta_{l}}\}_{l \in[L]}\|_2}{\textup{output margin of } (x,y)} \implies \textup{generalization error} \lesssim \frac{\sum_{l=1}^L \mathcal{C}_l}{\sqrt{n}} \sqrt{\frac{R(\theta)}{\textup{output margin}}}
\end{align*}
as $R(\theta)$ is an upper bound on the squared norm of the Jacobian at any global minimizer $\theta$.
We emphasize this bound is informal as we discarded the higher-order terms in controlling the all-layer margin, but it accurately reflects that the regularizer $R(\theta)$ lower bounds the all-layer margin $m_F$ up to higher-order terms. Therefore SGD with label noise implicitly regularizes the all-layer margin.


\section{Acknowledgements}
AD acknowledges support from a NSF Graduate Research Fellowship. TM acknowledges support of Google Faculty Award and NSF IIS 2045685. JDL acknowledges support of the ARO under MURI Award W911NF-11-1-0303,  the Sloan Research Fellowship, NSF CCF 2002272, and an ONR Young Investigator Award.

The experiments in this paper were performed on computational resources managed and supported by Princeton Research Computing, a consortium of groups including the Princeton Institute for Computational Science and Engineering (PICSciE) and the Office of Information Technology's High Performance Computing Center and Visualization Laboratory at Princeton University.

We would also like to thank Honglin Yuan and Jeff Z. HaoChen for useful discussions throughout various stages of the project.

\bibliography{myref}

\begin{thebibliography}{32}
\providecommand{\natexlab}[1]{#1}
\providecommand{\url}[1]{\texttt{#1}}
\expandafter\ifx\csname urlstyle\endcsname\relax
  \providecommand{\doi}[1]{doi: #1}\else
  \providecommand{\doi}{doi: \begingroup \urlstyle{rm}\Url}\fi

\bibitem[Arora et~al.(2019)Arora, Du, Hu, Li, Salakhutdinov, and
  Wang]{arora2019exact}
S.~Arora, S.~S. Du, W.~Hu, Z.~Li, R.~Salakhutdinov, and R.~Wang.
\newblock On exact computation with an infinitely wide neural net.
\newblock \emph{arXiv preprint arXiv:1904.11955}, 2019.

\bibitem[Biewald(2020)]{wandb}
L.~Biewald.
\newblock Experiment tracking with weights and biases, 2020.
\newblock URL \url{https://www.wandb.com/}.
\newblock Software available from wandb.com.

\bibitem[Blanc et~al.(2019)Blanc, Gupta, Valiant, and
  Valiant]{blanc2019implicit}
G.~Blanc, N.~Gupta, G.~Valiant, and P.~Valiant.
\newblock Implicit regularization for deep neural networks driven by an
  ornstein-uhlenbeck like process.
\newblock \emph{arXiv preprint arXiv:1904.09080}, 2019.

\bibitem[Cohen et~al.(2021)Cohen, Kaur, Li, Kolter, and
  Talwalkar]{cohen2021gradient}
J.~M. Cohen, S.~Kaur, Y.~Li, J.~Z. Kolter, and A.~Talwalkar.
\newblock Gradient descent on neural networks typically occurs at the edge of
  stability, 2021.

\bibitem[Du et~al.(2019)Du, Lee, Li, Wang, and Zhai]{du2019gradient}
S.~S. Du, J.~D. Lee, H.~Li, L.~Wang, and X.~Zhai.
\newblock Gradient descent finds global minima of deep neural networks, 2019.

\bibitem[{Falcon et al.}(2019)]{pytorchlightning2019}
W.~{Falcon et al.}
\newblock Pytorch lightning.
\newblock \emph{GitHub. Note:
  https://github.com/PyTorchLightning/pytorch-lightning}, 3, 2019.

\bibitem[Foret et~al.(2020)Foret, Kleiner, Mobahi, and
  Neyshabur]{foret2020sharpness}
P.~Foret, A.~Kleiner, H.~Mobahi, and B.~Neyshabur.
\newblock Sharpness-aware minimization for efficiently improving
  generalization.
\newblock \emph{arXiv preprint arXiv:2010.01412}, 2020.

\bibitem[Ge et~al.(2015)Ge, Huang, Jin, and Yuan]{ge2015escaping}
R.~Ge, F.~Huang, C.~Jin, and Y.~Yuan.
\newblock Escaping from saddle points—online stochastic gradient for tensor
  decomposition.
\newblock In \emph{Conference on Learning Theory}, pages 797--842, 2015.

\bibitem[Goyal et~al.(2017)Goyal, Doll{\'a}r, Girshick, Noordhuis, Wesolowski,
  Kyrola, Tulloch, Jia, and He]{goyal2017accurate}
P.~Goyal, P.~Doll{\'a}r, R.~Girshick, P.~Noordhuis, L.~Wesolowski, A.~Kyrola,
  A.~Tulloch, Y.~Jia, and K.~He.
\newblock Accurate, large minibatch sgd: Training imagenet in 1 hour.
\newblock \emph{arXiv preprint arXiv:1706.02677}, 2017.

\bibitem[Gunasekar et~al.(2017)Gunasekar, Woodworth, Bhojanapalli, Neyshabur,
  and Srebro]{gunasekar2017implicit}
S.~Gunasekar, B.~E. Woodworth, S.~Bhojanapalli, B.~Neyshabur, and N.~Srebro.
\newblock Implicit regularization in matrix factorization.
\newblock In \emph{Advances in Neural Information Processing Systems}, pages
  6151--6159, 2017.

\bibitem[Gunasekar et~al.(2018)Gunasekar, Lee, Soudry, and
  Srebro]{gunasekar2018characterizing}
S.~Gunasekar, J.~Lee, D.~Soudry, and N.~Srebro.
\newblock Characterizing implicit bias in terms of optimization geometry.
\newblock \emph{arXiv preprint arXiv:1802.08246}, 2018.

\bibitem[HaoChen et~al.(2020)HaoChen, Wei, Lee, and Ma]{haochen2020shape}
J.~Z. HaoChen, C.~Wei, J.~D. Lee, and T.~Ma.
\newblock Shape matters: Understanding the implicit bias of the noise
  covariance.
\newblock \emph{arXiv preprint arXiv:2006.08680}, 2020.

\bibitem[Hendrycks and Gimpel(2020)]{hendrycks2020gaussian}
D.~Hendrycks and K.~Gimpel.
\newblock Gaussian error linear units (gelus), 2020.

\bibitem[Jacot et~al.(2018)Jacot, Gabriel, and Hongler]{jacot2018neural}
A.~Jacot, F.~Gabriel, and C.~Hongler.
\newblock Neural tangent kernel: Convergence and generalization in neural
  networks.
\newblock In \emph{Advances in neural information processing systems}, pages
  8571--8580, 2018.

\bibitem[Jin et~al.(2019)Jin, Netrapalli, Ge, Kakade, and
  Jordan]{jin2019stochastic}
C.~Jin, P.~Netrapalli, R.~Ge, S.~M. Kakade, and M.~I. Jordan.
\newblock Stochastic gradient descent escapes saddle points efficiently.
\newblock \emph{arXiv preprint arXiv:1902.04811}, 2019.

\bibitem[Keskar et~al.(2016)Keskar, Mudigere, Nocedal, Smelyanskiy, and
  Tang]{keskar2016large}
N.~S. Keskar, D.~Mudigere, J.~Nocedal, M.~Smelyanskiy, and P.~T.~P. Tang.
\newblock On large-batch training for deep learning: Generalization gap and
  sharp minima.
\newblock \emph{arXiv preprint arXiv:1609.04836}, 2016.

\bibitem[Krizhevsky(2009)]{Krizhevsky09learningmultiple}
A.~Krizhevsky.
\newblock Learning multiple layers of features from tiny images.
\newblock Technical report, 2009.

\bibitem[LeCun et~al.(2012)LeCun, Bottou, Orr, and
  M{\"u}ller]{lecun2012efficient}
Y.~A. LeCun, L.~Bottou, G.~B. Orr, and K.-R. M{\"u}ller.
\newblock Efficient backprop.
\newblock In \emph{Neural networks: Tricks of the trade}, pages 9--48.
  Springer, 2012.

\bibitem[Li et~al.(2017)Li, Ma, and Zhang]{li2017algorithmic}
Y.~Li, T.~Ma, and H.~Zhang.
\newblock Algorithmic regularization in over-parameterized matrix sensing and
  neural networks with quadratic activations.
\newblock \emph{arXiv preprint arXiv:1712.09203}, 2017.

\bibitem[Li et~al.(2019)Li, Wei, and Ma]{li2019towards}
Y.~Li, C.~Wei, and T.~Ma.
\newblock Towards explaining the regularization effect of initial large
  learning rate in training neural networks.
\newblock In \emph{Advances in Neural Information Processing Systems}, pages
  11669--11680, 2019.

\bibitem[Liu et~al.(2019)Liu, Papailiopoulos, and Achlioptas]{liu2019bad}
S.~Liu, D.~Papailiopoulos, and D.~Achlioptas.
\newblock Bad global minima exist and sgd can reach them.
\newblock \emph{arXiv preprint arXiv:1906.02613}, 2019.

\bibitem[Mianjy et~al.(2018)Mianjy, Arora, and Vidal]{mianjy2018implicit}
P.~Mianjy, R.~Arora, and R.~Vidal.
\newblock On the implicit bias of dropout.
\newblock \emph{arXiv preprint arXiv:1806.09777}, 2018.

\bibitem[Moroshko et~al.(2020)Moroshko, Gunasekar, Woodworth, Lee, Srebro, and
  Soudry]{moroshko2020implicit}
E.~Moroshko, S.~Gunasekar, B.~Woodworth, J.~D. Lee, N.~Srebro, and D.~Soudry.
\newblock Implicit bias in deep linear classification: Initialization scale vs
  training accuracy.
\newblock \emph{Neural Information Processing Systems (NeurIPS)}, 2020.

\bibitem[Paszke et~al.(2019)Paszke, Gross, Massa, Lerer, Bradbury, Chanan,
  Killeen, Lin, Gimelshein, Antiga, Desmaison, Kopf, Yang, DeVito, Raison,
  Tejani, Chilamkurthy, Steiner, Fang, Bai, and Chintala]{pytorch2019}
A.~Paszke, S.~Gross, F.~Massa, A.~Lerer, J.~Bradbury, G.~Chanan, T.~Killeen,
  Z.~Lin, N.~Gimelshein, L.~Antiga, A.~Desmaison, A.~Kopf, E.~Yang, Z.~DeVito,
  M.~Raison, A.~Tejani, S.~Chilamkurthy, B.~Steiner, L.~Fang, J.~Bai, and
  S.~Chintala.
\newblock Pytorch: An imperative style, high-performance deep learning library.
\newblock In H.~Wallach, H.~Larochelle, A.~Beygelzimer, F.~d\textquotesingle
  Alch\'{e}-Buc, E.~Fox, and R.~Garnett, editors, \emph{Advances in Neural
  Information Processing Systems 32}, pages 8024--8035. Curran Associates,
  Inc., 2019.
\newblock URL
  \url{http://papers.neurips.cc/paper/9015-pytorch-an-imperative-style-high-performance-deep-learning-library.pdf}.

\bibitem[Shallue et~al.(2018)Shallue, Lee, Antognini, Sohl-Dickstein, Frostig,
  and Dahl]{shallue2018measuring}
C.~J. Shallue, J.~Lee, J.~Antognini, J.~Sohl-Dickstein, R.~Frostig, and G.~E.
  Dahl.
\newblock Measuring the effects of data parallelism on neural network training.
\newblock \emph{arXiv preprint arXiv:1811.03600}, 2018.

\bibitem[Smith et~al.(2017)Smith, Kindermans, Ying, and Le]{smith2017don}
S.~L. Smith, P.-J. Kindermans, C.~Ying, and Q.~V. Le.
\newblock Don't decay the learning rate, increase the batch size.
\newblock \emph{arXiv preprint arXiv:1711.00489}, 2017.

\bibitem[Soudry et~al.(2018)Soudry, Hoffer, Nacson, Gunasekar, and
  Srebro]{soudry2018implicit}
D.~Soudry, E.~Hoffer, M.~S. Nacson, S.~Gunasekar, and N.~Srebro.
\newblock The implicit bias of gradient descent on separable data.
\newblock \emph{The Journal of Machine Learning Research}, 19\penalty0
  (1):\penalty0 2822--2878, 2018.

\bibitem[Szegedy et~al.(2016)Szegedy, Vanhoucke, Ioffe, Shlens, and
  Wojna]{szegedy2016rethinking}
C.~Szegedy, V.~Vanhoucke, S.~Ioffe, J.~Shlens, and Z.~Wojna.
\newblock Rethinking the inception architecture for computer vision.
\newblock In \emph{Proceedings of the IEEE conference on computer vision and
  pattern recognition}, pages 2818--2826, 2016.

\bibitem[Vaskevicius et~al.(2019)Vaskevicius, Kanade, and
  Rebeschini]{vaskevicius2019implicit}
T.~Vaskevicius, V.~Kanade, and P.~Rebeschini.
\newblock Implicit regularization for optimal sparse recovery.
\newblock In \emph{Advances in Neural Information Processing Systems}, pages
  2968--2979, 2019.

\bibitem[Wei and Ma(2019)]{wei2019improved}
C.~Wei and T.~Ma.
\newblock Improved sample complexities for deep networks and robust
  classification via an all-layer margin.
\newblock \emph{arXiv preprint arXiv:1910.04284}, 2019.

\bibitem[Wen et~al.(2019)Wen, Luk, Gazeau, Zhang, Chan, and
  Ba]{wen2019interplay}
Y.~Wen, K.~Luk, M.~Gazeau, G.~Zhang, H.~Chan, and J.~Ba.
\newblock Interplay between optimization and generalization of stochastic
  gradient descent with covariance noise.
\newblock \emph{arXiv preprint arXiv:1902.08234}, 2019.

\bibitem[Woodworth et~al.(2020)Woodworth, Gunasekar, Lee, Moroshko, Savarese,
  Golan, Soudry, and Srebro]{woodworth2020kernel}
B.~Woodworth, S.~Gunasekar, J.~D. Lee, E.~Moroshko, P.~Savarese, I.~Golan,
  D.~Soudry, and N.~Srebro.
\newblock Kernel and rich regimes in overparametrized models.
\newblock \emph{arXiv preprint arXiv:2002.09277}, 2020.

\end{thebibliography}

\tableofcontents

\appendix
\section{Limitations}\label{sec:limitations}

In \Cref{sec:setup} we make three main assumptions: \Cref{asm:smooth} (smoothness), \Cref{asm:eta} (learning rate separation), and \Cref{asm:kl} (KL).

\Cref{asm:smooth} imposes the necessary smoothness conditions on $f$ to enable second order Taylor expansions of $\nabla L$. These smoothness conditions may not hold, e.g. if ReLU activations are used. This can be easily resolved by using a smooth activation like softplus or SiLU \citep{hendrycks2020gaussian}.

\Cref{asm:eta} is a very general assumption that lets $\eta$ be arbitrarily close to the maximum cutoff for gradient descent on a quadratic, $2/\ell$. However, for simplicity we do not track the dependence on $\nu$. This work therefore does not explain the ability of gradient descent to optimize neural networks at the "edge of stability" \citep{cohen2021gradient} when $\eta > 2/\ell$. Because we only assume \Cref{asm:smooth} of the model, our results must apply to quadratics as a special case where any $\eta > 2/\ell$ leads to divergence so this assumption is strictly necessary.

Although \Cref{asm:kl} is very general (see \Cref{lem:klanalytic}), the specific value of $\delta$ plays a large role in our \Cref{thm:sgdsp}. In particular, if $L$ satisfies \Cref{asm:kl} for any $\delta \ge 1/2$ then the convergence rate in $\epsilon$ is $\epsilon^{-6}$. However, this convergence rate can become arbitrarily bad as $\delta \to 0$. This rate is driven by the bound on $E(\theta^*)$ in \Cref{prop:hessiandecomp}, which does not contribute to implicit regularization and cannot be easily controlled. The error introduced at every step from bounding $E(\theta)$ at a minimizer $\theta^*$ is $\tilde O(\eta\sqrt{\lambda L(\theta^*)})$ and the size of each step in the regularized trajectory is $\eta \lambda \|\nabla R(\theta^*)\|$. Therefore if $L(\theta^*) = \Omega(\lambda)$, the error term is greater than the movement of the regularized trajectory. \Cref{sec:arbitrarynoise} repeats the argument in \Cref{sec:sketch:coupling} without making \Cref{asm:kl}. However, the cost is that you can no longer couple to a fixed potential $R$ and instead must couple to a changing potential $R_S$.

One final limitation is our definition of stationarity (\Cref{def:epsgamma}). As we discuss in \Cref{sec:epsgammasp}, this limitation is fundamental as the more direct statement of converging to an $\epsilon$-stationary point of $\frac{1}{\lambda} \tilde L$ is not true. Although we do not do so in this paper, if $\theta$ remains in a neighborhood of a fixed $\epsilon$-stationary point $\theta^*$ for a sufficiently long time, then it might be possible to remove this assumption by tail-averaging the iterates. However, this requires a much stronger notion of stationarity than first order stationarity which does not guarantee that $\theta$ remains in a neighborhood of $\theta^*$ for a sufficiently long time (e.g. it may converge to a saddle point which it then escapes).

\section{Missing Proofs}\label{sec:appendix:missingproofs}

\begin{proof}[Proof of \Cref{prop:hessiandecomp}]
	We have
	\begin{align}
		\nabla L(\theta) = \frac{1}{n} \sum_{i=1}^n (f_i(\theta)-y_i)\nabla f_i(\theta)
	\end{align}
	so
	\begin{align}
		\nabla^2 L(\theta) &= \frac{1}{n} \sum_{i=1}^n \left[\nabla f_i(\theta)\nabla f_i(\theta)^T + (f_i(\theta)-y_i)\nabla^2 f_i(\theta)\right] \\
		&= G(\theta) + E(\theta).
	\end{align}	
	In addition if we define $e_i(\theta) = f_i(\theta)-y_i$,
	\begin{align}
		\|E(\theta)\| &= \frac{1}{n}\norm{\sum_{i=1}^n e_i(\theta) \nabla^2 f_i(\theta)} \\
		&\le \frac{1}{n}\left[\sum_{i=1}^n e_i(\theta)^2\right]^{1/2}\left[\sum_{i=1}^n \|\nabla^2 f_i(\theta)\|^2\right]^{1/2} \\
		&= \frac{1}{n} \sqrt{2nL(\theta)}\cdot\sqrt{n \rho_f^2} \\
		&= \sqrt{2\rho_f L(\theta)} \\
		&= O(\sqrt{L(\theta)}).
	\end{align}
\end{proof}

\begin{definition}
	We define the quadratic variation $[\cdot]$ and quadratic covariation $[\cdot,\cdot]$ of a martingale $X$ to be
	\begin{align}
		[X]_k = \sum_{j < k} \|\xi_{j+1} - \xi_j\|^2 \qqtext{and} [X,X]_k = \sum_{j < k} (\xi_{j+1} - \xi_j)(\xi_{j+1} - \xi_j)^T.
	\end{align}
\end{definition}

\begin{lemma}[Azuma-Hoeffding]\label{lem:azuma}
	Let $X \in \mathbb{R}^d$ be a mean zero martingale with $[X]_k \le \sigma^2$. Then with probability at least $1-2de^{-\iota}$,
	\begin{align}
		\|X_k\| \le \sigma \sqrt{2 \iota}.
	\end{align}
\end{lemma}
\begin{corollary}\label{cor:azumacov}
	Let $X \in \mathbb{R}^d$ be a mean zero martingale with $[X,X]_k \preceq M$. Then with probability at least $1-2de^{-\iota}$,
	\begin{align}
		\|X_k\| \le \sqrt{2 \tr(M) \iota}.
	\end{align}
\end{corollary}

\begin{proof}[Proof of \Cref{prop:sketch:OUproperties}]
	A simple induction shows that
	\begin{align}
		\xi_k = \sum_{j < k} (I-\eta G)^j \epsilon_{k-j-1}^*. \label{eq:OUprop:sum}
	\end{align}
	Then
	\begin{align}
		\E[\xi_k\xi_k^T] &= \sum_{j < k} (I - \eta G)^j \eta \lambda G (I - \eta G)^j \\
		&= \eta \lambda G (2\eta G - \eta^2 G^2)^\dagger (I-(I - \eta G)^{2k}) \\
		&= \lambda \Pi_G (2 - \eta G)^{-1} (I-(I - \eta G)^{2k}).
	\end{align}
	Therefore $\E[\xi_k\xi_k^T] \preceq \frac{\eta}{\nu}I$ and $\E[\xi_k\xi_k^T] \to \lambda \Pi_G (2 - \eta G)^{-1}$.
	The partial sums of \Cref{eq:OUprop:sum} form a martingale with quadratic covariation bounded by
	\begin{align}
		&\sum_{j < k} (I-\eta G)^j \epsilon_{k-j-1}^*(\epsilon_{k-j-1}^*)^T (I-\eta G)^j \\
		&\preceq \sum_{j < k} (I-\eta G)^j n \eta \lambda G (I - \eta G)^j \\
		&= n \lambda \Pi_G (2 - \eta G)^{-1} (I-(I - \eta G)^{2k}) \\
		&\preceq \frac{n\lambda}{\nu} I
	\end{align}
	therefore by \Cref{cor:azumacov}, with probability at least $1-2de^{-\iota}$, $\|\xi_k\| \le \mathscr{X}$.
\end{proof}

We prove the following version of \Cref{prop:sketch:OUproperties} for the setting of \Cref{lem:convergence:coupling}:
\begin{proposition}\label{prop:OUnormglobal}
	Let $\xi_k$ be defined as in \Cref{def:globalOU}. Then for any $t \ge 0$, with probability $1-2de^{-\iota}$, $\|\xi_t\| \le \mathscr{X}$.
\end{proposition}
\begin{proof}
	For $k \in (T_m,T_{m+1}]$ define $G_k = G(\theta_m^*)$. Then we can write for any $k \ge 0$,
	\begin{align}
		\xi_{k+1} = (I - \eta G_k)\xi_k + \epsilon_k^*.
	\end{align}
	Let $\mathcal{F}_t = \sigma\{\mathcal{B}^{(k)}, \epsilon^{(k)} : k < t\}$. To each $k$ we will associate a martingale $\{X^{(k)}_j\}_{j \le k}$ adapted to $\mathcal{F}$ as follows. First let $X^{(k)}_0 = 0$. Then for all $k  \ge 0$ and all $j \ge 0$,
	\begin{align}
		X_{j+1}^{(k)} =
		\begin{cases}
			(I - \eta G_{k-1})X_j^{(k-1)} & j < k-1 \\
			X_j^{(k)} + \epsilon_{k-1}^* & j=k-1.
		\end{cases}
	\end{align}
	First we need to show $X^{(k)}$ is in fact a martingale. We will show this by induction on $k$. The base case of $k = 0$ is trivial. Next, it is easy to see that $X_j^{(k)} \in \mathcal{F}_j$. Therefore,
	\begin{align}
		\E[X_k^{(k)}|\mathcal{F}_{k-1}] = \E[X_{k-1}^{(k)}|\mathcal{F}_{k-1}] = X_{k-1}^{(k)}
	\end{align}
	and for $j < k - 1$:
	\begin{align}
		\E[X_{j+1}^{(k)}|\mathcal{F}_{j}] &= (I - \eta G_{k-1}) \E[X_{j+1}^{(k-1)} | \mathcal{F}_{j}] \\
		&= (I - \eta G_{k-1}) X_{j}^{(k-1)} \\
		&= X_{j}^{k}
	\end{align}
	where the second line followed from the induction hypothesis and the third line followed from the definition of $X_j^{(k)}$. Therefore $X^{(k)}$ is a martingale for all $k$.
	
	Next, I claim that $\xi_k = X_k^{(k)}$. We can prove this by induction on $k$. The base case is trivial as $\xi_0 = X^{(0)}_0 = 0$. Then,
	\begin{align}
		X_{k+1}^{(k+1)} &= X_{k}^{(k+1)} + \epsilon_k^* \\
		&= (I - \eta G_k)X_k^{(k)} + \epsilon_k^* \\
		&= \xi_{k+1}.
	\end{align}
	Finally, I claim that $[X^{(k)},X^{(k)}]_k \preceq \frac{n \lambda}{\nu}I$. We will prove this by induction on $k$. The base case is trivial as $X^{(0)}_0 = 0$. Then,
	\begin{align}
		[X^{(k+1)},X^{(k+1)}]_{k+1} &= [X^{(k+1)},X^{(k+1)}]_k + \epsilon_k^*(\epsilon_k^*)^T \\
		&= (I - \eta G_k)[X^{(k)},X^{(k)}]_{k}(I - \eta G_k) + \epsilon_k^*(\epsilon_k^*)^T \\
		&\preceq \frac{n \lambda}{\nu} \left[(I - \eta G_k)^2 + \eta \nu G_k\right] \\
		&\preceq \frac{n \lambda}{\nu} \left[I - G_k(2 - \eta G_k - \nu I)\right] \\
		&\preceq \frac{n \lambda}{\nu} I.
	\end{align}
	Therefore by \Cref{cor:azumacov}, $\|\xi_k\| \le \mathscr{X}$ with probability at least $1-2de^{-\iota}$.
\end{proof}

We will prove \Cref{prop:sketch:taylor2} and \Cref{prop:sketch:OUcov} in the more general setting of \Cref{lem:convergence:coupling}. For notational simplicity we will apply the Markov property and assume that $m = 0$. We define $\Delta = \Delta_0$ and $\theta^* = \theta_0^*$ and note that due to this time change that $\xi_0$ is not necessarily $0$. We define $v_k = \theta_k - \Phi_k(\theta^* + \Delta)$ and $r_k = \theta_k - \xi_k - \Phi_k(\theta^* + \Delta)$.

\begin{proof}[Proof of \Cref{prop:sketch:taylor2}]
	First, by \Cref{prop:OUnormglobal}, $\|\xi_t\| \le \mathscr{X}$ with probability at least $1-2de^{-\iota}$. Then note that for $k \le t$,
\begin{align}
	\|\theta_k - \theta^*\| \le \|\xi_k\| + \|r_k\| + \|\Phi_k(\theta^*+\Delta)-\theta^*\| = O(\mathscr{X}) \qqtext{and} \theta_k - \theta^* = \xi_k + O(\mathscr{M})
\end{align}
so Taylor expanding the update in \Cref{alg:sgdln} and \Cref{eq:regularizedtrajectory} to second order around $\theta^*$ and subtracting gives
\begin{align}
	v_{k+1} &= (I - \eta G)v_k + \epsilon_k^* + m_k + z_k \\
	&\qquad - \eta\left[\frac{1}{2}\nabla^3 L(\theta_{k}-\theta^*,\theta_{k}-\theta^*) - \frac{1}{2}\nabla^3 L(\Phi_{k}(\theta^*)-\theta^*,\Phi_{k}(\theta^*)-\theta^*)-\lambda \nabla R\right] \notag\\
	&\qquad + O(\eta \mathscr{X} (\sqrt{\mathscr{L}} + \mathscr{X}^2)) \notag \\
	&= (I - \eta G)v_k + \epsilon_k^* + m_k + z_k - \eta\left[\frac{1}{2}\nabla^3 L(\xi_k,\xi_k)-\lambda \nabla R\right] + O(\eta \mathscr{X}(\sqrt{\mathscr{L}} + \mathscr{M} + \mathscr{X}^2)) \notag.
\end{align}
Subtracting \Cref{eq:sketch:OUdef}, we have
\begin{align}
	r_{k+1} 
	&= (I - \eta G)r_k - \eta\left[\frac{1}{2}\nabla^3 L(\xi_k,\xi_k) - \lambda \nabla R \right] + m_k + z_k + O(\eta \mathscr{X}(\sqrt{\mathscr{L}} + \mathscr{M} + \mathscr{X}^2)) \\
	&= (I - \eta G)r_k - \eta\left[\frac{1}{2}\nabla^3 L(\xi_k,\xi_k) - \lambda \nabla R \right] + m_k + z_k + \tilde O(c^{5/2} \eta \lambda^{1+\delta/2}).
\end{align}
\end{proof}

\begin{proof}[Proof of \Cref{prop:sketch:OUcov}]
Note that for each $i \in \mathcal{B}^{(k)}$,
	\begin{align}
		\|\epsilon^{(k)}_i (\nabla f_i(\theta) - \nabla f_i(\theta^*))\| \le \sigma \rho_f \|\theta-\theta^*\|.
	\end{align}
	Therefore by \Cref{lem:azuma}, with probability $1-2de^{-\iota}$,
	\begin{align}
		\norm{\sum_{j < k} (I - \eta G)^j z_{k-j}} = O(\sqrt{\eta \lambda k \iota} \mathscr{X}).
	\end{align}
	Next, note that because $\|\nabla \ell_i(\theta)\|=O(L(\theta))$, by \Cref{lem:azuma}, with probability at least $1-2de^{-\iota}$,
	\begin{align}
		\norm{\sum_{j < k} (I - \eta G)^j m_{k-j}} = O(\sqrt{\eta \lambda k \iota} \sqrt{L(\theta)}).
	\end{align}
	Next, by a second order Taylor expansion around $\theta^*$ we have
	\begin{align}
		\sqrt{L(\theta)} \le O(\sqrt{\mathscr{L}} + \mathscr{X})
	\end{align}
	so
	\begin{align}
		r_{t+1} 
		&= -\eta \sum_{k \le t}(I - \eta G)^{t-k} \left[\frac{1}{2}\nabla^3 L(\xi_k,\xi_k) - \lambda \nabla R \right] \\
		&\qquad +  O\left(\sqrt{\eta \lambda t}\left(\sqrt{\mathscr{L}}+\mathscr{X}\right) + \eta t \mathscr{X}\left(\sqrt{\mathscr{L}} + \mathscr{M} + \mathscr{X}^2\right)\right) \notag \\
		&= -\eta \sum_{k \le t}(I - \eta G)^{t-k} \left[\frac{1}{2}\nabla^3 L(\xi_k,\xi_k) - \lambda \nabla R \right] +\tilde O\left(\frac{\lambda^{1/2+\delta/2}}{\sqrt{c}}\right).
	\end{align}
	Now we will turn to concentrating $\xi_k\xi_k^T$. We will use the shorthand $g_i = \nabla f_i(\theta^*)$. Let
	\begin{align}
		S^* = \lambda(2-\eta \nabla^2 L)^{-1},\qquad \bar S = \lambda(2-\eta G)^{-1}, \qqtext{and} S_k = \xi_k \xi_k^T.
	\end{align}
	It suffices to bound
	\begin{align}
		\eta \sum_{k \le t}(I - \eta G)^{t-k} \frac{1}{2} \nabla^3 L(S_k - S^*).
	\end{align}
	We can expand out $\nabla^3 L$ using the fact that $L$ is square loss to get
	\begin{align}
		\frac{1}{2} \nabla^3 L(S_k - S^*) = \frac{1}{n} \sum_{i=1}^n \left( H_i (S_k - S^*) g_i + \frac{1}{2}g_i \tr\left[(S_k - S^*) H_i\right] \right) + O(\sqrt{\mathscr{L}}\mathscr{X}^2),
	\end{align}
	so it suffices to bound the contribution of the first two terms individually. Starting with the second term, we have $\tr\left[(S_k - S^*) H_i\right] = O(\mathscr{X}^2)$, so by \Cref{contract:sum1},
	\begin{align}
		\eta\frac{1}{n} \sum_{i=1}^n \sum_{k \le t} (I - \eta G)^{t-k} g_i \tr\left[(S_k - S^*) H_i\right] = O\left(\sqrt{\eta t}\mathscr{X}^2 \right).
	\end{align}
	For the first term, note that
	\begin{align}
		S^* - \bar S = \lambda \left[(2-\eta \nabla^2 L)^{-1}\left((2-\eta G)-(2-\eta \nabla^2 L)\right)(2-\eta G)^{-1}\right] = O(\eta \lambda \sqrt{\mathscr{L}})
	\end{align}
	so this difference contributes at most $O(\eta^2 \lambda t \sqrt{\mathscr{L}}) = O(\eta t \mathscr{X} \sqrt{\mathscr{L}})$
	so it suffices to bound
	\begin{align}
		\frac{1}{n} \sum_{i=1}^n \sum_{k \le t} (I - \eta G)^{t-k} H_i (S_k - \bar S) g_i.
	\end{align}
	Now note that
	\begin{align}
		S_{k+1} = (I - \eta G)S_k(I - \eta G) + (I - \eta G)\xi_k (\epsilon_k^*)^T + \epsilon_k^* \xi_k (I - \eta G) + (\epsilon_k^*)(\epsilon_k^*)^T
	\end{align}
	and that\footnote{This identity directly follows from multiplying both sides by $2 - \eta G$ and the fact that all of these matrices commute .}
	\begin{align}
		\bar S = (I - \eta G)\bar S(I - \eta G) + \eta \lambda G.
	\end{align}
	Let $D_k = S_k - \bar S$. Then subtracting these two equations gives
	\begin{align}
		D_{k+1} = (I - \eta G)D_k(I - \eta G) + (I - \eta G)\xi_k (\epsilon_k^*)^T + \epsilon_k^* \xi_k (I - \eta G) + ((\epsilon_k^*)(\epsilon_k^*)^T - \eta \lambda G).
	\end{align}
	Let $W_k = (I - \eta G)\xi_k (\epsilon_k^*)^T + \epsilon_k^* \xi_k^T (I - \eta G)$ and let $Z_k = ((\epsilon_k^*)(\epsilon_k^*)^T - \eta \lambda G)$ so that
	\begin{align}
		D_{k+1} = (I - \eta G)D_k(I - \eta G) + W_k + Z_k.
	\end{align}
	Then,
	\begin{align}
		D_k = (I - \eta G)^k D_0 (I - \eta G)^k + \sum_{j < k} (I - \eta G)^{k-j-1} (W_j + Z_j) (I - \eta G)^{k-j-1}.
	\end{align}
	Substituting the first term gives
	\begin{align}
		\eta\frac{1}{n} \sum_{i=1}^n \sum_{k \le t} (I - \eta G)^{t-k} H_i (I - \eta G)^k D_0 (I - \eta G)^k g_i = O(\sqrt{\eta t}\mathscr{X}^2)
	\end{align}
	so we are left with the martingale part in the second term. The final term to bound is therefore
	\begin{align}
		\eta\frac{1}{n} \sum_{i=1}^n \sum_{k \le t} (I - \eta G)^{t-k} H_i \left[\sum_{j < k} (I - \eta G)^{k-j-1} (W_j + Z_j) (I - \eta G)^{k-j-1}\right] g_i.\label{eq:wjzj}
	\end{align}
	We can switch the order of summations to get 
	\begin{align}
		\eta\frac{1}{n} \sum_{i=1}^n \sum_{j \le t} \sum_{k = j+1}^t (I - \eta G)^{t-k} H_i (I - \eta G)^{k-j-1} (W_j + Z_j) (I - \eta G)^{k-j-1}g_i.
	\end{align}
	Now if we extract the inner sum, note that 
	\begin{align}
		\sum_{k = j+1}^t (I - \eta G)^{t-k} H_i (I - \eta G)^{k-j-1} (W_j + Z_j) (I - \eta G)^{k-j-1}g_i
	\end{align}
	is a martingale difference sequence. Recall that
	\begin{align}
		\epsilon_j^* = \frac{\eta}{B} \sum_{l \in \mathcal{B}^{(j)}} \epsilon^{(j)}_l g_l
	\end{align}
	First, isolating the $W$ term, we get
	\begin{align}
		&\sum_{k = j+1}^t (I - \eta G)^{t-k} H_i (I - \eta G)^{k-j} \xi_j (\epsilon_j^*)^T (I - \eta G)^{k-j-1}g_i \\
		&\qquad+ \sum_{k = j+1}^t (I - \eta G)^{t-k} H_i (I - \eta G)^{k-j-1} \epsilon_j^* \xi_j^T (I - \eta G)^{k-j}g_i. \notag \\
		&= \frac{\eta}{B} \sum_{l \in \mathcal{B}^{(j)}} \epsilon^{(j)}_l \Bigl[\sum_{k = j+1}^t (I - \eta G)^{t-k} H_i (I - \eta G)^{k-j} \xi_j g_l^T (I - \eta G)^{k-j-1}g_i \label{eq:OUcov:W}\\
		&\qquad+ \sum_{k = j+1}^t (I - \eta G)^{t-k} H_i (I - \eta G)^{k-j-1} g_l \xi_j^T (I - \eta G)^{k-j}g_i\Bigr] \notag.
	\end{align}
	The inner sums are bounded by $O(\mathscr{X}\eta^{-1})$ by \Cref{contract:sumsqapart}. Therefore by \Cref{lem:azuma}, with probability at least $1-2de^{-\iota}$, the contribution of the $W$ term in \Cref{eq:wjzj} is at most $O(\sqrt{\eta \lambda k \iota}\mathscr{X}) = O(\sqrt{\eta k}\mathscr{X}^2)$. The final remaining term to bound is the $Z$ term in \eqref{eq:wjzj}. We can write the inner sum as
	\begin{align}
		\frac{\eta\lambda}{B^2}\sum_{k = j+1}^t (I - \eta G)^{t-k} H_i (I - \eta G)^{k-j-1} \left(\frac{1}{\sigma^2}\sum_{l_1,l_2 \in \mathcal{B}^{(k)}} \epsilon^{(j)}_{l_1} \epsilon^{(j)}_{l_2} g_{l_1} g_{l_2}^T  - G\right) (I - \eta G)^{k-j-1}g_i
	\end{align}
	which by \Cref{contract:sumsqapart} is bounded by $O(\lambda)$. Therefore by \Cref{lem:azuma}, with probability at least $1-2de^{-\iota}$, the full contribution of $Z$ to \Cref{eq:wjzj} is $O(\eta \lambda \sqrt{t \iota}) = O(\sqrt{\eta t}\mathscr{X}^2)$. Putting all of these bounds together we get with probability at least $1-10de^{-\iota}$,
	\begin{align}
		\|r_{t+1}\| &= O\left[\sqrt{\eta \mathscr{T}}\mathscr{X}(\sqrt{\mathscr{L}} + \mathscr{X}) + \eta \mathscr{T} \mathscr{X}(\sqrt{\mathscr{L}} + \mathscr{M} + \mathscr{X}^2)\right] \\
		&= \tilde O\left(\frac{\lambda^{1/2+\delta/2}}{\sqrt{c}}\right).
	\end{align}
\end{proof}

The following lemma is necessary for some of the proofs below:
\begin{lemma}\label{lem:lossdecrease}
	Assume that $L(\theta) \le \mathscr{L}$. Then for any $k \ge 0$, $L(\Phi_k(\theta)) \le \mathscr{L}$.
\end{lemma}
\begin{proof}
	By induction it suffices to prove this for $k=1$. Let $\theta' = \Phi_1(\theta)$. First consider the case when
	\begin{align}
		\|\nabla L(\theta')\| \le \left(\frac{\mathscr{L}}{\mu}\right)^{1/(1+\delta)}.		
	\end{align}
Then by \Cref{asm:kl}, $L(\theta') \le \mathscr{L}$ so we are done. Otherwise, note that
\begin{align}
	\|\nabla L(\theta)\| &\ge \|\nabla L(\theta')\| - \ell \|\theta - \theta'\| \\
	&\ge \Omega(c\lambda) - \eta \ell \|\nabla L(\theta)\|
\end{align}
so $\|\nabla L(\theta)\| \ge \Omega(c \lambda)$ and therefore $\|\nabla \tilde L(\theta)\| \ge \Omega(c \lambda)$ Then by the standard descent lemma,
\begin{align}
	L(\theta') &\le L(\theta) - \eta \nabla \tilde{L}(\theta)^T \nabla L(\theta) + \frac{\eta^2 \ell}{2} \|\nabla \tilde{L}(\theta)^2\| \\
	&\le L(\theta) - \frac{\eta}{2}(2 - \eta \ell) \|\nabla \tilde{L}(\theta)\|^2 + O(\eta \lambda \|\nabla L(\theta)\|) \\
	&= L(\theta) - \frac{\eta\nu}{2} \|\nabla \tilde{L}(\theta)\|^2 + O(\eta \lambda \|\nabla L(\theta)\|) \\
\end{align}
and for $c$ sufficiently large, the second term is larger than the third so $L(\theta') \le L(\theta) \le \mathscr{L}$.
\end{proof}

We break the proof of \Cref{lem:convergence:1step} into a sequence of propositions. The idea behind \Cref{lem:convergence:1step} is to consider the trajectory $\Phi_k(\theta_m^*)$, for $k \le \mathscr{T}$. First, we want to carefully pick $\tau_m$ so that $\eta \sum_{k < \tau_m} \|\nabla \tilde L(\Phi_k(\theta_m^*))\|$ is sufficiently large to decrease the regularized loss $\tilde{L}$ but sufficiently small to be able to apply \Cref{lem:convergence:coupling}:
\begin{proposition}\label{prop:picktau}
	In the context of \Cref{lem:convergence:1step}, if $\theta_{T_m}$ is not an $(\epsilon,\gamma)$-stationary point, there exists $\tau_m \le \mathscr{T}$ such that:
	\begin{align}
		5\mathscr{M} \ge \eta \sum_{k < \tau_n} \|\nabla \tilde L(\Phi_k(\theta_n^*))\| \ge 4\mathscr{M}.
	\end{align}
\end{proposition}

We can use this to lower bound the decrease in $\tilde{L}$ from $\theta_m^*$ to $\Phi_{\tau_m}(\theta_m^*)$:
\begin{proposition}\label{prop:reglossdescent}
	$\tilde L(\Phi_{\tau_m}(\theta_m^*)) \le \tilde L(\theta_m^*) - 8\frac{\mathscr{D}^2}{\eta \nu \tau_m}$.
\end{proposition}
We now bound the increase in $\tilde{L}$ from $\Phi_{\tau_m}(\theta_m^*)$ to $\theta_{m+1}^*$. This requires relating the regularized trajectories starting at $\theta_m^*$ and $\theta_m^* + \Delta_m$. The following proposition shows that the two trajectories converge in the directions where the eigenvalues of $G(\theta_m^*)$ are large:
\begin{proposition}\label{prop:regcontract} Let $G = G(\theta_m^*)$ and let $\tau_m$ be chosen as in \Cref{prop:picktau}. Then, $\theta_{m+1}^* - \Phi_{\tau_m}(\theta_m^*) = (I - \eta G)^{\tau_m} \Delta_m + r$ where $\|r\| = O(\eta \tau_m \mathscr{M}^2)$ and $\|r\|_G^2 = O(\eta \tau_m \mathscr{M}^4)$.
\end{proposition}

Substituting the result in \Cref{prop:regcontract} into the second order Taylor expansion of $\tilde{L}$ centered at $\Phi_{\tau_m}(\theta_m^*)$ gives:
\begin{proposition}\label{prop:reglossascent}
	$\tilde L(\theta_{m+1}^*) \le \tilde L(\Phi_{\tau_m}(\theta_m^*)) + 7\frac{\mathscr{D}^2}{\eta \nu \tau_m}$
\end{proposition}
Combining \Cref{prop:reglossdescent,prop:reglossascent}, we have that
\begin{align}
	\tilde{L}(\theta_{m+1}^*) - \tilde L(\theta_m^*) \le  - \frac{\mathscr{D}^2}{\eta \nu \tau_m} \le -\mathscr{F}.\label{eq:chaindescent}
\end{align}
where the last line follows from $\tau_m \le \mathscr{T}$ and the definition of $\mathscr{F}$. Finally, the following proposition uses this bound on $\tilde{L}$ and \Cref{asm:kl} to bound $L(\theta_{m+1}^*)$:
\begin{proposition}\label{prop:chainloss}
	$L(\theta_{m+1}^*) \le \mathscr{L}$.
\end{proposition}
The following corollary also follows from the choice of $\tau_m$,  \Cref{prop:regcontract}, and \Cref{lem:convergence:coupling}:
\begin{corollary}\label{cor:chainerrors}
	$\|\Phi_{\tau_m}(\theta_m^* + \Delta_m) - \theta_m^*\| \le 8\mathscr{M}$ and with probability at least $1-8d\tau_m e^{-\iota}$, $\|\Delta_{m+1}\| \le \mathscr{D}$.
\end{corollary}

The proof of \Cref{lem:convergence:1step} follows directly from \Cref{eq:chaindescent}, \Cref{prop:chainloss}, and \Cref{cor:chainerrors}. The proofs of the above propositions can be found below:

\begin{proof}[Proof of \Cref{prop:picktau}]
	First, assume that
	\begin{align}
		\eta \sum_{k < \mathscr{T}} \|\nabla \tilde{L}(\Phi_k(\theta_m^*))\| \ge 4\mathscr{M}.
	\end{align}
	Then we can upper bound each element in this sum by
	\begin{align}
		\eta \|\nabla \tilde{L}(\Phi_k(\theta_m^*))\| \le \eta \|\nabla L(\Phi_k(\theta_m^*))\| + \eta \lambda \|\nabla R(\Phi_k(\theta_m^*))\|.
	\end{align}
	Note that
	\begin{align}
		\|\nabla L(\theta)\| &= \norm{\frac{1}{n} \sum_{i=1}^n (f_i(\theta)-y_i)\nabla f_i(\theta)} \\
		&\le \frac{1}{n}\left[\sum_{i=1}^n (f_i(\theta)-y_i)^2\right]^{1/2}\left[\sum_{i=1}^n \|\nabla f_i(\theta)\|^2\right]^{1/2} \\
		&\le \sqrt{2\ell_f L(\theta)}
	\end{align}
	and because $\nabla R$ is bounded,
	\begin{align}
		\eta \|\nabla \tilde{L}(\Phi_k(\theta_m^*))\| \le O(\eta L(\Phi_k(\theta_m^*)) + \eta \lambda).
	\end{align}
	Then by \Cref{lem:lossdecrease},
	\begin{align}
		\eta \|\nabla \tilde{L}(\Phi_k(\theta_m^*))\| \le O(\eta \sqrt{\mathscr{L}} + \eta \lambda) \le \mathscr{M}
	\end{align}
	for sufficiently large $c$.	Therefore there must exist $\tau_m$ such that
	\begin{align}
		5\mathscr{M} \ge \eta \sum_{k < \mathscr{T}} \|\nabla \tilde{L}(\Phi_k(\theta_m^*))\| \ge 4\mathscr{M}.
	\end{align}
	
	Otherwise,
	\begin{align}
	\eta \sum_{k < \mathscr{T}} \|\nabla \tilde{L}(\Phi_k(\theta_m^*))\| < 4\mathscr{M}.	
	\end{align}
	Therefore there must exist some $k$ such that
	\begin{align}
		\frac{1}{\lambda}\|\nabla \tilde{L}(\Phi_k(\theta_m^*))\| < \frac{4\mathscr{M}}{\eta\lambda \mathscr{T}} = O(\lambda^{\delta/2}\iota^{3/2}) \le \epsilon
	\end{align}
	by the choice of $\lambda$ in \Cref{thm:sgdsp}. In addition,
	\begin{align}
		\|\theta_{T_m} - \Phi_k(\theta_m^*)\| \le \|\theta_{T_m} - \theta_m^*\| + 4\mathscr{M} \le \mathscr{X} + \mathscr{D} + 4\mathscr{M} \le \gamma
	\end{align}
	again by the choice of $\lambda$. Therefore $\theta_{T_m}$ is an $(\epsilon,\gamma)$-stationary point.
\end{proof}

\begin{proof}[Proof of \Cref{prop:reglossdescent}]
	We have by the standard descent lemma
	\begin{align}
		\tilde L(\Phi_{\tau_m}(\theta_m^*)) &\le -\frac{\eta \nu}{2}\sum_{k < \tau_m} \|\nabla \tilde L(\Phi_{\tau_k}(\theta_m^*))\|^2 \\
		&\le -\frac{\eta \nu}{2 \tau_m} \left[\sum_{k < \tau_m} \|\nabla \tilde L(\Phi_{\tau_k}(\theta_m^*))\|\right]^2 \\
		&\le -\frac{\nu \mathscr{M}^2}{2 \eta \tau_m} \\
		&= -8\frac{\mathscr{D}^2}{\eta \nu \tau_m}.
	\end{align}
\end{proof}

\begin{proof}[Proof of \Cref{prop:regcontract}]
	Let $v_k = \Phi_k(\theta_m^* + \Delta_m) - \Phi_k(\theta_m^*)$, so that $v_0 = \Delta_m$ and let $r_k = v_k - (I - \eta G)^{\tau_m}$ so that $r_0 = 0$. Let $C$ be a sufficiently large absolute constant. We will prove by induction that $r_k \le C\eta \tau_m \mathscr{M}^2$. Note that
	\begin{align}
		\|\Phi_k(\theta_m^* + \Delta_m) - \theta_m^*\| &\le \|\Delta_m\| + \|\Phi_k(\theta_m^*) - \theta_m^*\| + \|v_k\| \\
		&\le O(\mathscr{M} + \eta \mathscr{T} \mathscr{M}^2) \\
		&= O(\mathscr{M})
	\end{align}
	because of the values chosen for $\mathscr{M}$, $\mathscr{T}$. Therefore Taylor expanding around $\theta_m^*$ gives:
	\begin{align}
		v_{k+1} &= v_k - \eta\left[\nabla \tilde L(\Phi_k(\theta_m^* + \Delta_m)) - \nabla \tilde L(\Phi_k(\theta_m^*))\right] \\
		&= v_k - \eta \nabla^2 \tilde L v_k + O(\eta \mathscr{M}^2) \\
		&= (I - \eta G)v_k + O(\eta \mathscr{M}^2 + \eta\lambda \mathscr{M} + \eta \sqrt{\mathscr{L}}\mathscr{M}) \\
		&= (I - \eta G)v_k + s_k
	\end{align}
	where $\|s_k\| =  O(\eta \mathscr{M}^2)$ by the definition of $\mathscr{M}$. Therefore
	\begin{align}
		v_k = (I - \eta G)^k \Delta_m + O(\eta k \mathscr{M}^2).
	\end{align}
	In addition,
	\begin{align}
		r_k = \sum_{j < k} (I - \eta G)^j s_{k-j}
	\end{align}
	so if $g_i = \nabla f_i(\theta_m^*)$,
	\begin{align}
		r_k^T G r_k &= \frac{1}{n} \sum_{i=1}^n (s_{k-j}^T \sum_{j < k} (I - \eta G)^j g_i)^2 \\
		&\le O(\eta^2 \mathscr{M}^4) \frac{1}{n} \sum_{i=1}^n \|\sum_{j < k} (I - \eta G)^j g_i\|^2 \\
		&= O(\eta \tau_m \mathscr{M}^4)
	\end{align}
	by \Cref{contract:sum1}, so we are done.
\end{proof}

We will need the following lemma before the next proof:
\begin{lemma}\label{lem:normdecrease} For any $k < \tau_m$,
	\begin{align}
	\|\nabla \tilde L(\Phi_k(\theta_m^*))\| \ge 11\|\nabla \tilde{L}(\Phi_{\tau_m}(\theta_m^*))\|/12.	
	\end{align}
\end{lemma}
\begin{proof}
	\begin{align}
		\nabla \tilde{L}(\Phi_{k+1}(\theta)) = (I - \eta \nabla^2 L(\Phi_k(\theta)))\nabla \tilde{L}(\Phi_{k}(\theta)) + O(\eta^2 \|\nabla \tilde{L}(\Phi_{k}(\theta))\|^2)
	\end{align}
	By \Cref{lem:lossdecrease} and \Cref{prop:hessiandecomp},
	\begin{align}
		\|I - \eta \nabla^2 L(\Phi_k(\theta))\| \le 1 + \eta \sqrt{2 \rho_f \mathscr{L}}.
	\end{align}
	In addition, 
	\begin{align}
		\|\nabla \tilde{L}(\Phi_k(\theta))\| = O(\lambda + \sqrt{\mathscr{L}})
	\end{align}
	so
	\begin{align}
		\|\nabla \tilde{L}(\Phi_{k+1}(\theta))\| \le (1 + O(\mathscr{L})) \|\nabla \tilde{L}(\Phi_{k}(\theta))\|.
	\end{align}
	Therefore,
	\begin{align}
		\|\nabla \tilde{L}(\Phi_{\tau_m}(\theta))\| &\le (1 + O(\mathscr{L}))^{\tau_m-k} \|\nabla \tilde{L}(\Phi_{k}(\theta))\| \\
		&\le \exp(O(\mathscr{T} \mathscr{L})) \|\nabla \tilde{L}(\Phi_{k}(\theta))\| \\
		&\le 12\|\nabla \tilde{L}(\Phi_{k}(\theta))\|/11
	\end{align}
	for sufficiently large $c$.
\end{proof}

\begin{proof}[Proof of \Cref{prop:reglossascent}]
	Let $v = \theta_{m+1} - \Phi_{\tau_m}(\theta_m^*) = (I - \eta G)^{\tau_m}\Delta_m + r$ where by \Cref{prop:regcontract}, $\|r\| = O(\eta \tau_m \mathscr{M}^2)$, $G = G(\theta_m^*)$, and $r^T G r = O(\eta \tau_m \mathscr{M}^4)$. Then,
	\begin{align}
		&\tilde L(\theta_{m+1}^*) - \tilde L(\Phi_{\tau_m}(\theta_m^*)) \\
		&\le \|v\|\|\|\tilde L(\Phi_{\tau_m}(\theta_m^*)\| + \frac{1}{2} v^T \nabla^2 \tilde L(\Phi_{\tau_m}(\theta_m^*)) v + O(\|v\|^3) \\
		&\le \|v\| \|\tilde L(\Phi_{\tau_m}(\theta_m^*)\| + \frac{1}{2} v^T G v + O(\mathscr{D}^2(\mathscr{D}+\sqrt{\mathscr{L}}+\lambda)) \\
		&\le \|v\| \|\tilde L(\Phi_{\tau_m}(\theta_m^*)\| + \Delta_m^T (I - \eta G)^{\tau_m} G (I - \eta G)^{\tau_m} \Delta_m + r^T G r \\
		&\qquad+ O(\mathscr{D}^2(\mathscr{D}+\sqrt{\mathscr{L}}+\lambda)).
	\end{align}
	By \Cref{prop:regcontract},
	\begin{align}
		\|v\| \le \mathscr{D} + O(\eta \tau_m \mathscr{M}^2) = \mathscr{D} + \mathscr{D}\cdot O\left(\frac{\lambda^{\delta/2}}{\sqrt{c}}\right) \le 11\mathscr{D}/10
	\end{align}
	for sufficiently large $c$. Therefore by \Cref{lem:normdecrease} and \Cref{prop:picktau},
	\begin{align}
		\|v\|\|\tilde L(\Phi_{\tau_m}(\theta_m^*)\| \le \mathscr{D}\frac{6}{5\tau_m} \sum_{k < \tau_m} \|\tilde L(\Phi_k(\theta_m^*)\| \le \frac{6\mathscr{D}^2}{\eta \nu \tau_m}.
	\end{align}
	By \Cref{contract:G},
	\begin{align}
		\Delta_m^T (I - \eta G)^{\tau_m} G (I - \eta G)^{\tau_m} \Delta_m \le \frac{\mathscr{D}^2}{2\eta \nu \tau_m}.
	\end{align}
	By \Cref{prop:regcontract},
	\begin{align}
		r^T G r = O(\eta \tau_m \mathscr{M}^4) &= \frac{\mathscr{D}^2}{\eta \nu \tau_m} O(\eta^2 \tau_m^2 \mathscr{D}^2) \\
		&= \frac{\mathscr{D}^2}{\eta \nu \tau_m} O\left(\frac{\lambda^{\delta}}{c}\right) \\
		&\le \frac{\mathscr{D}^2}{4\eta \nu \tau_m}
	\end{align}
	for sufficiently large $c$. Finally, the remainder term is bounded by
	\begin{align}
		\frac{\mathscr{D}^2}{\eta \nu \tau_m} \cdot O(\eta \tau_m \mathscr{D}) \le \frac{\mathscr{D}^2}{4\eta \nu \tau_m}
	\end{align}
	for sufficiently large $c$ for the same reason as above. Putting it all together,
	\begin{align}
		\tilde L(\theta_{m+1}^*) - \tilde L(\Phi_{\tau_m}(\theta_m^*)) \le \frac{6\mathscr{D}^2}{\eta \nu \tau_m} + \frac{\mathscr{D}^2}{2\eta \nu \tau_m} + \frac{\mathscr{D}^2}{4\eta \nu \tau_m} + \frac{\mathscr{D}^2}{4\eta \nu \tau_m} = \frac{7\mathscr{D}^2}{\eta \nu \tau_m}.
	\end{align}
\end{proof}

\begin{proof}[Proof of \Cref{prop:chainloss}]
	Assume otherwise for the sake of contradiction. Because $\nabla R$ is Lipschitz, $R(\theta_{m+1}^*)-R(\theta_m^*) = O(\mathscr{M})$. Therefore by \Cref{eq:chaindescent},
	\begin{align}
		L(\theta_{m+1}^*) \le L(\theta_m^*) - \frac{\mathscr{D}^2}{\eta \nu \tau_m} + O(\lambda \mathscr{M}).
	\end{align}
	Therefore we must have $\mathscr{D} = O(\eta \lambda \tau_m)$ so by \Cref{prop:picktau} and \Cref{lem:normdecrease} we have that $\|\nabla \tilde L (\Phi_{\tau_m}(\theta_m^*))\| = O(\lambda)$ and because $\lambda \nabla R = O(\lambda)$ we must have $\|\nabla L(\Phi_{\tau_m}(\theta_m^*))\| = O(\lambda)$. Therefore by \Cref{asm:kl},
	\begin{align}
		L(\Phi_{\tau_m}(\theta_m^*)) = O(\lambda^{1+\delta}).
	\end{align}
	Then by the same arguments as in \Cref{prop:reglossascent}, we can Taylor expand around $\Phi_{\tau_m}(\theta_m^*)$ to get
	\begin{align}
		&L(\theta_{m+1}^*) - L(\Phi_{\tau_m}(\theta_m^*)) \\
		&\le \|\nabla L(\Phi_{\tau_m}(\theta_m^*))\|v + \frac{1}{2} v^T \nabla^2 L(\Phi_{\tau_m}(\theta_m^*)) v + O(\mathscr{D}^3) \\
		&\le O\left(\lambda \mathscr{D} + \frac{\mathscr{D}^2}{\eta \tau} + \mathscr{D}^3\right) \\
		&\le O(\lambda^{1+\delta})
	\end{align}
	because $\delta \le 1/2$. Therefore $L(\theta_{m+1}^*) = O(\lambda^{1+\delta}) \le \mathscr{L}$ for sufficiently large $c$.
\end{proof}

\section{Reaching a global minimizer with NTK}
\label{sec:NTK}

It is well known that overparameterized neural networks in the kernel regime trained by gradient descent reach global minimizers of the training loss \cite{jacot2018neural,du2019gradient}. In this section we describe how to extend the proof in \cite{du2019gradient} to show that SGD with label noise (\Cref{alg:sgdln}) converges to a neighborhood of a global minimizer $\theta^*$ as required by \Cref{thm:sgdsp}. We will use the following lemma from \cite{du2019gradient}:
\begin{lemma}[\cite{du2019gradient}, Lemma B.4]
	There exists $R = \tilde O(\sqrt{m} \lambda_0)$ such that every $\theta \in B_R(\theta_0)$ satisfies $\lambda_{min}(\mathcal{G}(\theta)) \ge \lambda_0/2$ where $\mathcal{G}_{ij}(\theta) = \langle \nabla f_i(\theta),\nabla f_j(\theta) \rangle$ and $\lambda_0$ is the minimum eigenvalue of the infinite width NTK matrix.
\end{lemma}

Let $\xi_0 = 0$ and $\theta_0^* = \theta_0$. We will define $\xi_k,\theta_k^*$ iteratively as follows:

\begin{align}
	\xi_{k+1} = (I - \eta G(\theta_k^*))\xi_k + \epsilon_k \qqtext{and} \theta_{k+1}^* = \theta_k^* - \eta \nabla L(\theta_k^*) - \eta E(\theta_k^*)(\theta_k - \theta_k^*) - z_k.
\end{align}

Let $v_k = \theta_k - \theta_k^*$ and let $r_k = v_k - \xi_k$. We will prove by induction that for all $t \le T = \frac{4\log\left[L(\theta_0)\lambda_0/\lambda^2\right]}{\eta \lambda_0}$ we have $\|r_k\| \le \mathscr{D}$. The base case follows from $r_0 = 0$. For $k \ge 0$ we have
\begin{align*}
	v_{k+1}
	&= v_k - \eta[\nabla L(\theta_k) - \nabla L(\theta_k^*) - E(\theta_k^*)v_k] + \epsilon_k^* \\
	&= (I - \eta G)v_k + \epsilon_k^* + O(\eta \mathscr{X}^2)
\end{align*}
so
\begin{align*}
	r_{k+1} &= (I - \eta G)r_k + O(\eta \mathscr{X}^2) \\
	&= O(\eta T \mathscr{X}^2) \\
	&= O(\mathscr{D})
\end{align*}
which completes the induction. Therefore it suffices to show that the loss of $\theta_T^*$ is small. We have
\begin{align*}
	\E[L(\theta_{k+1}^*)|\theta_k^*] &\le L(\theta_k^*) - \nabla L(\theta_k^*)^T [\eta \nabla L + \eta E(\theta_k^*)v_k] \\
	&\quad+ O[\eta^2 \|\nabla L\|^2 + \eta^2 \|E(\theta_k^*)\|^2\|v_k\|^2 + \|\epsilon_k - \epsilon_k^*\|^2] \\
	&\le L(\theta_k^*) - \frac{\eta}{4} \|\nabla L(\theta_k^*)\|^2 + O\left(\eta \|E(\theta_k^*)\|^2\|v_k\|^2 + \|\epsilon_k - \epsilon_k^*\|^2\right)
\end{align*}
where the last line follows from Young's inequality. Therefore,
\begin{align*}
	L(\theta_{k+1}^*) &\le L(\theta_k^*) - \frac{\eta}{4} \|\nabla L(\theta_k^*)\|^2 + O\left(\eta L(\theta_k^*) \mathscr{X}^2 + \eta \lambda \mathscr{X}^2\right) \\
	&= L(\theta_k^*) - \frac{\eta}{4} \|\nabla L(\theta_k^*)\|^2 + \tilde O\left(\eta \lambda L(\theta_k^*) + \eta \lambda^2 \right) \\
	&= (1 + \tilde O(\eta \lambda))L(\theta_k^*) - \frac{\eta}{4} \|\nabla L(\theta_k^*)\|^2 + \tilde O\left(\eta \lambda^2 \right).
\end{align*}
Let $J$ be the Jacobian of $f$ and $e$ be the vector of residuals. Then $\nabla L = Je$. Now so long as $\|\theta_k^* - \theta_0\| \le R$,
\begin{align}	
	\|\nabla L(\theta_k^*)\|^2 = e(\theta_k^*)^T J(\theta_k^*)^T J(\theta_k^*) e(\theta_k^*) \ge \lambda_0 \|e(\theta_k^*)\|^2 = 2\lambda_0 L(\theta_k^*).
\end{align}
Therefore,
\begin{align*}
	L(\theta_{k+1}^*) &\le \left(1 - \frac{\eta\lambda_0}{2} + \tilde O(\eta \lambda)\right)L(\theta_k^*) + \tilde O\left(\eta \lambda^2 \right).
\end{align*}
Now for $\lambda = \tilde O(\lambda_0)$,
\begin{align*}
	L(\theta_{k+1}^*) &\le \left(1 - \frac{\eta\lambda_0}{4}\right)L(\theta_k^*) + \tilde O\left(\eta \lambda^2 \right)
\end{align*}
so
\begin{align*}
	L(\theta_{T}^*) &\le \left(1 - \frac{\eta\lambda_0}{4}\right)^T L(\theta_0) + \tilde O\left(\frac{\lambda^2}{\lambda_0}\right) \\
	&\le \tilde O\left(\frac{\lambda^2}{\lambda_0}\right) = O(\lambda^{1+\delta})
\end{align*}
for small $\lambda$ by the choice of $T$. It only remains to check that $\|\theta_k^* - \theta_0\| \le R$. Note that
\begin{align}
	\|\theta_k^* - \theta_0\| &\le \eta \sum_{j < k} \|\nabla L(\theta_j^*)\| + \tilde O(\eta T \sqrt{\lambda} +\sqrt{\eta T \lambda^2}) \\
	&\le \tilde O\left(\eta \sum_{j \le k} \sqrt{L(\theta_j^*)} + \sqrt{\lambda}\right) \\
	&\le \tilde O\left(\frac{L(\theta_0)}{\lambda_0}\right)
\end{align}
so for $m \ge \tilde \Omega(1/\lambda_0^4)$ we are done.

Note that a direct application of \Cref{thm:sgdsp} requires starting $\xi$ at $0$. However, this does not affect the proof in any way and the $\xi$ from this proof can simply be continued as in \Cref{lem:convergence:coupling}.

Finally, note that although $\|\frac{1}{\lambda} \nabla \tilde L(\theta)\| = O(1/\sqrt{m})$ at any global minimizer, \Cref{thm:sgdsp} guarantees that for any $\lambda > 0$ we can find a point $\theta$ where $\|\frac{1}{\lambda} \nabla \tilde L(\theta)\| \lesssim \lambda^{\delta/2} \ll 1/\sqrt{m}$, as $m$ only needs to be larger than a fixed constant depending on the condition number of the infinite width NTK kernel.

\section{Additional Experimental Details}\label{sec:appendix:experiments}

The model used in our experiments is ResNet18 with GroupNorm instead of BatchNorm to maintain independence of sample gradients when computed in a batch. We used a fixed group size of 32.

For the full batch initialization, we trained ResNet18 on the CIFAR10 training set (50k images, 5k per class) \citep{Krizhevsky09learningmultiple}, with cross entropy loss. CIFAR10 images are provided under an MIT license. We trained using SGD with momentum with $\eta = 1$ and $\beta = 0.9$ for $2000$ epochs. We used learning rate warmup starting at $0$ which linearly increased until $\eta = 1$ at epoch $600$ and then it decayed using a cosine learning rate schedule to $0$ between epochs $600$ and $2000$. We also used a label smoothing value of $0.2$ (non-randomized) so that the expected objective function is the same for when we switch to SGD with label flipping (see \Cref{sec:classification}). The final test accuracy was $76\%$.

For the adversarial initialization, we first created an augmented adversarial dataset as follows. We duplicate every image in CIFAR10 $10\times$, for a total of 500k images. In each image, we randomly zero out $10\%$ of the pixels in the image and we assign each of the 500k images a random label. We trained ResNet18 to interpolate this dataset without label smoothing with the following hyperparameters: $\eta = 0.01$, $300$ epochs, batch size $256$. Starting from this initialization we ran SGD on the true dataset with $\eta = 0.01$ and a label smoothing value of $0.2$ with batch size $256$ for $1000$ epochs. The final test accuracy was $48\%$.

For the remaining experiments starting at these two initializations we ran both with and without momentum (see \Cref{fig:experiments_m} for the results with momentum) for $1000$ epochs per run. We used a fixed batch size of $256$ and varied the maximum learning rate $\eta$. We used learning rate warmup by linearly increasing the learning rate from $0$ to the max learning rate over $300$ epochs, and we kept the learning rate constant from epochs $300$ to $1000$. The regularizer was estimated by computing the strength of the noise in each step and then averaging over an epoch. More specifically, we compute the average of $\|\nabla \hat L^{(k)}(\theta_k) - \nabla L^{(k)}(\theta_k)\|^2$ over an epoch and then renormalize by the batch size.

The experiments were run on NVIDIA P100 GPUs through Princeton Research Computing. Code was written in Python using PyTorch \citep{pytorch2019} and PyTorch Lightning \citep{pytorchlightning2019}, and experiments were logged using Wandb \citep{wandb}. Code can be found at \url{https://github.com/adamian98/LabelNoiseFlatMinimizers}.


\section{Extension to Classification}\label{sec:classification}

\subsection{Proof of \Cref{thm:sgdlssp}}

The proof of \Cref{thm:sgdlssp} is virtually identical to that of \Cref{thm:sgdsp}. First we make a few simplifications without loss of generality:

First note that if we scale $l$ by $\frac{1}{\alpha}$ and $\eta$ by $\alpha$ then the update in \Cref{alg:sgdls} i remain constant. In addition, $\frac{1}{\lambda} \tilde L = \frac{1}{\lambda} L + R$ remains constant. Therefore it suffices to prove \Cref{thm:sgdlssp} in the special case when $\alpha = 1$.

Next note that without loss of generality we can replace each $f_i$ with $y_i f_i$ and set all of the true labels $y_i$ to $1$. Therefore from now on we will simply speak of $f_i$.

Let $\{\tau_m\}$ be a sequence of coupling times and $\{\theta_m^*\}$ a sequence of reference points. Let $T_m = \sum_{j < m} \tau_m$. Then for $k \in [T_m,T_{m+1})$, if $L^{(k)}$ denotes true value of the loss on batch $\mathcal{B}^{(k)}$, we can decompose the loss as
\begin{align}
	\theta_{k+1}
	&= \theta_k - \underbrace{\eta \nabla L(\theta_k)}_\text{gradient descent} - \underbrace{\eta[\nabla L^{(k)}(\theta_k) - \nabla L(\theta_k)]}_\text{minibatch noise} + \underbrace{\frac{\eta}{B} \sum_{i \in \mathcal{B}^{(k)}} \epsilon_i^{(k)} \nabla f_i(\theta_k)}_\text{label noise}
\end{align}
where
\begin{align}
\epsilon_i^{(k)} = \begin{cases}
 	-p[l'(f_i(\theta_k)) + l'(-f_i(\theta_k))] & \sigma_i^{(k)} = 1 \\
 	(1-p)[l'(f_i(\theta_k)) + l'(-f_i(\theta_k))] & \sigma_i^{(k)} = -1.
 \end{cases}
\end{align}
We define
\begin{align}
	\epsilon_k = \frac{\eta}{B} \sum_{i \in \mathcal{B}^{(k)}} \epsilon_i^{(k)} \nabla f_i(\theta_k) \qqtext{and} m_k = \nabla L^{(k)}(\theta_k) - \nabla L(\theta_k).
\end{align}
We decompose $\epsilon_k = \epsilon_k^* + z_k$ where
\begin{align}
	\epsilon_k^* = \frac{\eta}{B} \sum_{i \in \mathcal{B}^{(k)}} \epsilon_i^{(k)*} \nabla f_i(\theta_m^*) \qqtext{where} \epsilon_i^{(k)*} = \begin{cases}
 	-p[l'(f_i(c)) + l'(-f_i(c))] & \sigma_i^{(k)} = 1 \\
 	(1-p)[l'(f_i(c)) + l'(-f_i(c))] & \sigma_i^{(k)} = -1
 \end{cases}
\end{align}
and $z_k = \epsilon_k - \epsilon_k^*$. Note that $\epsilon_k^*$ has covariance $\eta \lambda G(\theta_m^*)$. We define $\xi_0 = 0$ and for $k \in [T_m,T_{m+1})$,
\begin{align}
	\xi_{k+1} = (I - \eta G(\theta_m^*))\xi_k + \epsilon_k^*.
\end{align}

Then we have the following version of \Cref{prop:OUnormglobal}:
\begin{proposition}\label{prop:OUnormglobal:ls}
	Let $\mathscr{X} = \sqrt{\max\left(\frac{p}{1-p},\frac{1-p}{p}\right) \cdot \frac{2 \lambda d \iota}{\nu}}$. Then for any $t \ge 0$, with probability $1-2de^{-\iota}$, $\|\xi_t\| \le \mathscr{X}$.
\end{proposition}
\begin{proof}
	Let $P = \max\left(\frac{p}{1-p},\frac{1-p}{p}\right)$. Define the martingale sequence $X^{(k)}_j$ as in \Cref{prop:OUnormglobal}. I claim that $[X^{(k)},X^{(k)}]_k \preceq \frac{n \lambda P}{\nu}I$. We will prove this by induction on $k$. The base case is trivial as $X^{(0)}_0 = 0$. Then,
	\begin{align}
		[X^{(k+1)},X^{(k+1)}]_{k+1} &= [X^{(k+1)},X^{(k+1)}]_k + \epsilon_k^*(\epsilon_k^*)^T \\
		&= (I - \eta G_k)[X^{(k)},X^{(k)}]_{k}(I - \eta G_k) + \epsilon_k^*(\epsilon_k^*)^T \\
		&\preceq \frac{n \lambda P}{\nu} \left[(I - \eta G_k)^2 + \eta \nu G_k\right] \\
		&\preceq \frac{n \lambda P}{\nu} \left[I - G_k(2 - \eta G_k - \nu I)\right] \\
		&\preceq \frac{n \lambda P}{\nu} I.
	\end{align}
	Therefore by \Cref{cor:azumacov} we are done.
\end{proof}

Define $\iota, \mathscr{D},\mathscr{M},\mathscr{T}, \mathscr{L}$ as in \Cref{lem:sketch:coupling}. Then we have the following local coupling lemma:

\begin{lemma}\label{lem:coupling:ls}
	Assume $f$ satisfies \Cref{asm:smooth}, $\eta$ satisfies \Cref{asm:eta}, and $l$ satisfies \Cref{asm:quadraticapprox}. Let $\Delta_m = \theta_{T_m} - \xi_{T_m} - \theta_m^*$ and assume that $\|\Delta_m\| \le \mathscr{D}$ and $L(\theta_m^*) \le \mathscr{L}$ for some $0 < \delta \le 1/2$. Then for any $\tau_m \le \mathscr{T}$ satisfying $\max_{k \in [T_m,T_{m+1})} \|\Phi_{k-T_m}(\theta_m^* + \Delta_m) - \theta_m^*\| \le 8\mathscr{M}$, with probability at least $1-10 d \tau_m e^{-\iota}$ we have simultaneously for all $k \in (T_m,T_{m+1}]$,
	\begin{align}
		\norm{\theta_k - \xi_k - \Phi_{k-T_m}(\theta_m^* + \Delta_m)} \le \mathscr{D}, \qquad \E[\xi_k]=0, \qqtext{and} \|\xi_k\| \le \mathscr{X}.
	\end{align}
\end{lemma}

The proof of \Cref{lem:coupling:ls} follows directly from the following decompositions:
\begin{proposition}\label{prop:hessiandecomp:ls}
Let $\nabla^2 L = \nabla^2 L(\theta_m^*)$, $\nabla^3 L = \nabla^3 L(\theta_m^*)$, $G = G(\theta_m^*)$, $f_i = f_i(\theta_m^*)$, $g_i = \nabla f_i(\theta_m^*)$, $H_i = \nabla^2 f_i(\theta_m^*)$. Then,
\begin{align}
	\nabla^2 L = G + O(\sqrt{\mathscr{L}}) \qqtext{and} \frac{1}{2} \nabla^3 L(S) = \frac{1}{n} \sum_i H_i vv^T g_i + g_i O(\|v\|^2) + O(\|v\|^2\sqrt{\mathscr{L}}).
\end{align}
\end{proposition}
\begin{proof}
	First, note that
	\begin{align}
		\nabla^2 L &= \frac{1}{n} \sum_i l''(f_i) g_ig_i^T + l'(f_i)H_i \\
		&= G + \ell \sqrt{\frac{1}{n}\sum_i [l''(f_i)-l''(c)]^2} + \rho_f \sqrt{\frac{1}{n}\sum_i [l'(f_i)]^2} \\
		&= G + O(\sqrt{\mathscr{L}})
	\end{align}
	by \Cref{asm:quadraticapprox}.
	Next,
	\begin{align}
		\frac{1}{2}\nabla^3 L(v,v) &= \frac{1}{2n} \sum_i 2l''(f_i) H_i vv^T g_i + g_i [l'''(f_i) (g_i^T v)^2 + l''(f_i)v^T H_i v] + O(l'(f_i)) \\
		&= \frac{1}{n} \sum_i H_i vv^T g_i + g_i O(\|v\|^2) + O(\|v\|^2\sqrt{\mathscr{L}}).
	\end{align}
\end{proof}

These are the exact same decompositions used \Cref{prop:sketch:taylor2} and \Cref{prop:sketch:OUcov}, so \Cref{lem:coupling:ls} immediately follows. In addition, as we never used the exact value of the constant in $\mathscr{X}$ in the proof of \Cref{thm:sgdsp}, the analysis there applies directly as well showing that we converge to an $(\epsilon,\gamma)$-stationary point and proving \Cref{thm:sgdlssp}.

\subsection{Verifying \Cref{asm:quadraticapprox}}\label{sec:quadraticapprox}

We verify \Cref{asm:quadraticapprox} for the logistic loss, the exponential loss, and the square loss and derive the corresponding values of $c,\sigma^2$ found in \Cref{table:differentlosses}.

\subsubsection{Logistic Loss}

For logistic loss, we let $l(x) = \log(1+e^{-x})$, and $\bar l(x) = pl(-x) + (1-p)l(x)$. Then

\begin{align}
	\bar l'(x) = \frac{pe^x - (1-p)}{1+e^x}
\end{align}
which is negative when $x < \log \frac{1-p}{p}$ and positive when $x > \log \frac{1-p}{p}$ so it is minimized at $c = \log \frac{1-p}{p}$. To show the quadratic approximation holds at $c$, it suffices to show that $\bar l'''(x)$ is bounded. We have $\bar l''(x) = \frac{e^x}{(1+e^x)^2}$ and
\begin{align}
	\bar l'''(x) = \frac{e^x(1-e^x)}{(1+e^x)^3} < \frac{1}{4}
\end{align}
so we are done. Finally, to calculate the strength of the noise at $c$ we have
\begin{align}
	\sigma^2 = p(1-p)(l'(c) + l'(-c))^2 = p(1-p)(-p + p-1)^2 = p(1-p).
\end{align}

\subsubsection{Exponential Loss}

We have $l(x) = e^{-x}$ and $\bar l(x) = pl(-x) + (1-p)l(x)$. Then,
\begin{align}
	\bar l'(x) = pe^x - (1-p)e^{-x}
\end{align}
which is negative when $x < \frac{1}{2}\log \frac{1-p}{p}$ and positive when $x > \frac{1}{2}\log \frac{1-p}{p}$ so it is minimized at $c = \frac{1}{2}\log \frac{1-p}{p}$. Then we can compute
\begin{align}
\bar l(c+x) = 2\sqrt{p(1-p)}\cosh(x) \ge 2\sqrt{p(1-p)} + \sqrt{p(1-p)}x^2 = L^* + \sqrt{p(1-p)}x^2
\end{align}
because $\cosh x \ge 1 + \frac{x^2}{2}$. Finally to compute the strength of the noise we have
\begin{align}
	\sigma^2 = p(1-p)(l'(c) + l'(-c))^2 = p(1-p)\left(-\sqrt{\frac{p}{1-p}} - \sqrt{\frac{1-p}{p}}\right)^2 = 1.
\end{align}

\subsubsection{Square Loss}

We have $l(x) = \frac{1}{2}(1-x)^2$ and $\bar l(x) = pl(-x) + (1-p)l(x)$. Then,
\begin{align}
	\bar l(x) = \frac{1}{2} [p(1+x)^2 + (1-p)(1-x)^2] = \frac{1}{2}[x^2 + x(4p-2) + 1]
\end{align}
which is a quadratic minimized at $c=1-2p$. The quadratic approximation trivially holds and the strength of the noise is:
\begin{align}
	\sigma^2 = p(1-p)(l'(c) + l'(-c))^2 = p(1-p)(-2p -2(1-p))^2 = 4p(1-p)
\end{align}

\section{Arbitrary Noise}

\subsection{Proof of \Cref{prop:arbitrarynoise}}

We follow the proof of \Cref{lem:convergence:coupling}. First, let $\epsilon_k = \sqrt{\eta \lambda} \Sigma^{1/2}(\theta_k) x_k$ with $x_k \sim N(0,I)$ and define $\epsilon_k^* = \sqrt{\eta \lambda}\Sigma^{1/2}(\theta^*) x_k$ and $z_k = \epsilon_k - \epsilon_k^*$. Let $H = \nabla^2 L(\theta^*)$, $\Sigma = \Sigma(\theta^*)$, and $\nabla R_S = \nabla R_S(\theta^*)$. Let $\alpha$ be the smallest nonzero eigenvalue of $H$. Unlike in \Cref{lem:sketch:coupling}, we will omit the dependence on $\alpha$.

First we need to show $S$ exists. Consider the update
\begin{align}
	S \leftarrow (I - \eta H)S(I - \eta H) + \eta \lambda \Sigma(\theta^*)
\end{align}
Restricted to the span of $H$, this is a contraction so it must converge to a fixed point. In fact, we can write this fixed point in a basis of $H$ explicitly. Let $\{\lambda_i\}$ be the eigenvalues of $H$. The following computation will be performed in an eigenbasis of $H$. Then the above update is equivalent to:
\begin{align}
	S_{ij} = (1 - \eta \lambda_i)(1 - \eta \lambda_j) S_{ij} + \eta \lambda \Sigma_{ij}(\theta^*).
\end{align}
Therefore if $\lambda_i,\lambda_j \ne 0$ we can set
\begin{align}
	S_{ij} = \frac{\lambda \Sigma_{ij}(\theta^*)}{\lambda_i + \lambda_j - \eta \lambda_i \lambda_j}.
\end{align}
Otherwise we set $S_{ij} = 0$. Note that this is the unique solution restricted to $\mathrm{span}(H)$. Next, define the Ornstein-Uhlenbeck process $\xi$ as follows:
\begin{align}
	\xi_{k+1} = (I - \eta H)\xi_k + \epsilon_k^*.
\end{align}
Then note that
\begin{align}
	\xi_k = \sum_{j < k} (I - \eta H)^j \epsilon_{k-j}
\end{align}
so $\xi$ is Gaussian with covariance
\begin{align}
	\eta \lambda \sum_{j < k} (I - \eta H)^j \Sigma (I - \eta H)^j.
\end{align}
This is bounded by
\begin{align}
	C \eta \lambda \sum_{j < k} (I - \eta H)^j H (I - \eta H)^j \preceq C \lambda (2 - \eta H)^{-1} \preceq \frac{C\lambda}{\nu} I
\end{align}
so by \Cref{cor:azumacov}, $\|\xi_k\| \le \mathscr{X}$ with probability $1-2de^{-\iota}$. Define $v_k = \theta_k - \Phi_k(\theta_0)$ and $r_k = \theta_k - \xi_k - \Phi_k(\theta_0)$. We will prove by induction that $\|r_t\| \le \mathscr{D}$ with probability at least $1-8d t e^{-\iota}$. First, with probability $1-2de^{-\iota}$, $\|\xi_t\| \le \mathscr{X}$. In addition, for $k \le t$,
\begin{align}
	\|\theta_k - \theta^*\| \le 9\mathscr{D} + \mathscr{X} = O(\mathscr{X}).
\end{align}
Therefore from the second order Taylor expansion:
\begin{align}
	r_{k+1} = (I - \eta H)r_k - \eta \left[\frac{1}{2}\nabla^3 L(\xi_k,\xi_k) -  \lambda \nabla R_S \right] + z_k + O(\eta \mathscr{X}(\mathscr{D} + \mathscr{X}^2)).
\end{align}
Because $z_k$ is Gaussian with covariance bounded by $O(\eta \lambda \mathscr{X}^2)$ by the assumption that $\Sigma^{1/2}$ is Lipschitz, we have by the standard Gaussian tail bound that its contribution after summing is bounded by $\sqrt{\eta \lambda \mathscr{X} k\iota}$ with probability at least $1-2de^{-\iota}$ so summing over $k$ gives
\begin{align}
	r_{t+1} = - \eta \sum_{k \le t} (I - \eta H)^{t-k} \left[\frac{1}{2}\nabla^3 L(\xi_k,\xi_k) -  \lambda \nabla R_S \right] + O(\sqrt{\eta \lambda t}\mathscr{X} + \eta t \mathscr{X}(\mathscr{D} + \mathscr{X}^2)).
\end{align}
Now denote $S_k = \xi_k\xi_k^T$. Then we need to bound
\begin{align}
	\eta \sum_{k \le t} (I - \eta H)^{t-k} \nabla^3 L(S_k - S).
\end{align}
Let $D_k = S_k - S$. Then plugging this into the recurrence for $S_k$ gives
\begin{align}
	D_{k+1} = (I - \eta H)D_k(I - \eta H) + W_k + Z_k	
\end{align}
where
\begin{align}
	W_k = (I - \eta H)\xi_k (\epsilon_k^*)^T + \epsilon_k^* (\xi_k)^T (I - \eta H) \qqtext{and} Z_k = \epsilon_k^* (\epsilon_k^*)^T - \eta \lambda \Sigma.
\end{align}
Then,
\begin{align}
	D_k = (I - \eta H)^k S (I - \eta H)^k + \sum_{j < k} (I - \eta H)^{k-j-1} (W_j + Z_j) (I - \eta H)^{k-j-1}
\end{align}
so we need to bound
\begin{align}
	\eta \sum_{k \le t} (I - \eta H)^{t-k} \nabla^3 L\left[(I - \eta H)^k S (I - \eta H)^k + \sum_{j < k} (I - \eta H)^{k-j-1} (W_j + Z_j) (I - \eta H)^{k-j-1}
\right].
\end{align}
Because $S$ is in the span of $H$,
\begin{align}
	\norm{\eta \sum_{k < t} (I - \eta H)^{t-k} \nabla^3 L\left[(I - \eta H)^k S (I - \eta H)^k\right]} = O(\eta \lambda) \norm{\sum_{k \le t} (I - \eta H)^k \Pi_H} = O(\lambda/\alpha) = O(\lambda).
\end{align}
where $\Pi_H$ is the projection onto $H$. We switch the order of summation for the next two terms to get 
\begin{align}
	\eta \sum_{j \le t} \sum_{k = j+1}^t (I - \eta H)^{t-k} \nabla^3 L \left[(I - \eta H)^{k-j-1} (W_j + Z_j) (I - \eta H)^{k-j-1}\right].
\end{align}
Note that conditioned on $\epsilon_l^*$, $l < j$, the $W_j$ part of the inner sum is Gaussian with variance bounded by $O(\eta \lambda \mathscr{X}^2)$ so by \Cref{lem:subgaussian}, with probability at least $1-2de^{-\iota}$, the contribution of $W$ is bounded by $O(\sqrt{\eta \lambda t \iota}\mathscr{X})$.

For the $Z$ term, we will define a truncation parameter $r$ to be chosen later. Then define $\bar x_j = x_j[\|x_j\| \le r]$ where $x_j \sim N(0,I)$ is defined above. Define $\bar X = \E[\bar x_j \bar x_j^T]$. Then we can decompose the $Z$ term into:
\begin{align}
	&\eta^2 \lambda \sum_{j \le t} \sum_{k = j+1}^t (I - \eta H)^{t-k} \nabla^3 L \left[(I - \eta H)^{k-j-1}\Sigma^{1/2} \left(x_jx_j^T - \bar x_j\bar x_j^T\right)(\Sigma^{1/2})^T (I - \eta H)^{k-j-1}\right] \\
	+~&\eta^2 \lambda \sum_{j \le t} \sum_{k = j+1}^t (I - \eta H)^{t-k} \nabla^3 L \left[(I - \eta H)^{k-j-1}\Sigma^{1/2} \left(\bar x_j\bar x_j^T - \bar X_j\right)(\Sigma^{1/2})^T (I - \eta H)^{k-j-1}\right] \\
	+~&\eta^2 \lambda \sum_{j \le t} \sum_{k = j+1}^t (I - \eta H)^{t-k} \nabla^3 L \left[(I - \eta H)^{k-j-1}\Sigma^{1/2} \left(\bar X - I\right)(\Sigma^{1/2})^T (I - \eta H)^{k-j-1}\right].
\end{align}
With probability $1-2dte^{-r^2/2}$ we can assume that $x_j x_j^T = \bar x_j \bar x_j^T$ for all $j \le t$ so the first term is zero. For the second term the inner sum is bounded by $O(r^2 \eta^{-1})$ and has variance bounded by $O(\eta^{-2})$ by the same arguments as above. Therefore by Bernstein's inequality, the whole term is bounded by $O(\eta \lambda \sqrt{t \iota} + r^2 \eta \lambda \iota)$ with probability $1-2de^{-\iota}$. Finally, to bound the third term note that
\begin{align}
	\norm{\bar X - I}_F = \E[\|x_j\|^2 [\|x_j\| > r]] \le \sqrt{\E[\|x_j\|^4]\Pr[\|x_j\| > r]} \le (d+1)\sqrt{2d}e^{-r^2/4}.
\end{align}
Therefore the whole term is bounded by $O(\eta \lambda t e^{-r^2/4})$. Finally, pick $r = \sqrt{4 \iota \log \mathscr{T}}$. Then the final bound is
\begin{align}
r_{t+1} &\le O\left(\sqrt{\eta \mathscr{T}} \mathscr{X}^2 + \eta \mathscr{T} \mathscr{X}(\mathscr{D} + \mathscr{X}^2)\right)	\\
&= O\left(\frac{\lambda^{3/4}\iota^{1/4}}{c}\right) \\
&\le \mathscr{D}
\end{align}
for sufficiently large $c$. This completes the induction.

\subsection{SGD Cycling}\label{sec:cycling}

Let $\theta = (x,y,z_1,z_2,z_3,z_4)$. We will define a set of functions $f_i$ as follows:
\begin{align}
	f_1(\theta) = (1-y)z_1-1, && f_2(\theta) = (1-y)z_1+1, && f_3(\theta) = (1+y)z_2-1, && f_4(\theta) = (1+y)z_2+1, \\
	f_5(\theta) = (1-x)z_3-1, && f_6(\theta) = (1-x)z_4+1, && f_7(\theta) = (1+x)z_4-1, && f_8(\theta) = (1+x)z_4+1, \\
	f_9(\theta) = (1-x)z_1, && f_{10}(\theta)=(1+x)z_2, && f_{11}(\theta) = (1+y)z_3, && f_{12}(\theta) = (1-y)z_4, \\
	f_{13}(\theta) = x^2+y^2-1
\end{align}
and we set all labels $y_i = 0$. Then we verify empirically that if we run minibatch SGD with the loss function $\ell_i(\theta) = \frac{1}{2}(f_i(\theta)-y_i)^2$ then $(x,y)$ cycles counter clockwise over the set $x^2+y^2=1$:

\begin{figure}[H]
	\centering
	\includegraphics[width=\linewidth]{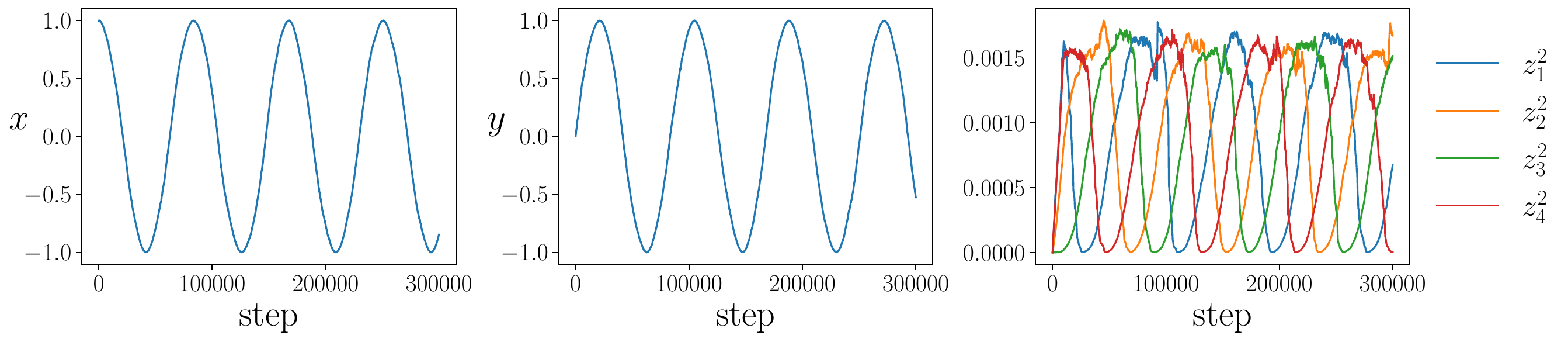}
	\caption{\textbf{Minibatch SGD can cycle.} We initialize at the point $\theta = (1,0,0,0,0,0)$. The left column shows $x$ over time which follows a cosine curve. The middle column shows $y$ over time which follows a sine curve. Finally, the right column shows moving averages of $z_i^2$ for $i=1,2,3,4$, which periodically grow and shrink depending on the current values of $x,y$.}
	\label{fig:cycling}
\end{figure}

The intuition for the definition of $f$ above is as follows. When $x = 1$ and $y=0$, due to the constraints from $f_9$ to $f_12$, only $z_1$ can grow to become nonzero. Then locally, $f_1 = z_1 - 1$ and $f_2 = z_1 + 1$ so this will cause oscillations in the $z_1$ direction, so $S$ will concentrate in the $z_1$ direction which will bias minibatch SGD towards decreasing the corresponding entry in $\nabla^2 L(\theta)$ which is proportional to $(1-x)^2 + 2(1-y)^2$, which means it will increase $y$. Similarly when $x=0,y=1$ there is a bias towards decreasing $x$, when $x=-1,y=0$ there is a bias towards decreasing $y$, and when $x=0,y=-1$ there is a bias towards increasing $x$. Each of these is handled by a different Ornstein Uhlenbeck process $z_i$. $f_{13}$ ensures that $\theta$ remains on $x^2+y^2=1$ throughout this process. This cycling is a result of minimizing a rapidly changing potential and shows that the implicit bias of minibatch SGD cannot be recovered by coupling to a fixed potential.

\section{Weak Contraction Bounds and Additional Lemmas}\label{sec:contract}

Let $\{g_i\}_{i \in [n]}$ be a collection of $n$ vectors such that $\|g_i\|_2 \le \ell_f$ for all $i$ and let $G = \frac{1}{n} \sum_i g_i g_i^T$. Let the eigenvalues of $G$ be $\lambda_1,\ldots,\lambda_n$, and assume that $\eta$ satisfies \Cref{asm:eta}. Then we have the following contraction bounds:
\begin{lemma}\label{contract:G}
	\begin{align}
	\norm{(I - \eta G)^\tau G} \le \frac{1}{\eta \nu \tau} = O\left(\frac{1}{\eta \tau}\right)
	\end{align}
\end{lemma}
\begin{lemma}\label{contract:power}
	\begin{align}
	\norm{(I - \eta G)^\tau g_{i}} = O\left(\sqrt{\frac{1}{\eta \tau}}\right)
	\end{align}
\end{lemma}
\begin{lemma}\label{contract:sum1}
	\begin{align}
	\sum_{k < \tau} \norm{(I - \eta G)^k g_{i_k}} = O\left(\sqrt{\frac{\tau}{\eta}}\right)
	\end{align}
\end{lemma}
\begin{lemma}\label{contract:sumsq}
	\begin{align}
	\sum_{k < \tau} \|(I - \eta G)^k g_{i_k}\|^2 = O\left(\frac{1}{\eta}\right)
	\end{align}
\end{lemma}
\begin{lemma}\label{contract:sumsqapart}
	\begin{align}
	\sum_{k < \tau} \|(I - \eta G)^k g_{i_k}\|\|(I - \eta G)^k g_{j_k}\| = O\left(\frac{1}{\eta}\right)
	\end{align}
\end{lemma}
\begin{lemma}\label{contract:sumG}
	\begin{align}
	\sum_{k < \tau} \|(I - \eta G)^k G\|_2 = O\left(\frac{1}{\eta}\right)
	\end{align}
\end{lemma}
\begin{proof}[Proof of \Cref{contract:G}]
	\begin{align}
		\norm{(I - \eta G)^\tau G} &= \max_i \abs{1-\eta \lambda_i}^\tau \lambda_i \le \max\left(\frac{\eta \lambda_i \tau}{\eta \tau} \exp(-\eta \lambda_i \tau),\ell (1-\nu)^\tau\right) \\
		&\le \max\left(\frac{1}{e \eta \tau}, \frac{1}{\eta\nu\tau} (\eta \ell) [\nu \tau] e^{-\nu \tau}\right)\\
		&\le \frac{1}{\eta \nu \tau}
		&= O\left(\frac{1}{\eta \tau}\right)
	\end{align}
	where we used that the function $xe^{-x} < \frac{1}{e}$ is bounded.
\end{proof}

\begin{proof}[Proof of \Cref{contract:power}]
	Note that
	\begin{align}
	\|(I - \eta G)^\tau g_i\|^2 &= \tr\left[(I - \eta G)^\tau g_i g_i^T (I - \eta G)^\tau \right] \\
	&\le n\tr((I - \eta G)^{2\tau} G)\\
	&= n\sum_i \lambda_i (1-\eta \lambda_i)^{2\tau} \\
	&\le n\sum_i \max\left(\lambda_i \exp(-2\eta \lambda_i \tau),\ell (1-\nu)^\tau\right) \\
	&= O\left(\frac{1}{\eta \tau}\right),
	\end{align}
	where we used the fact that the function $x e^{-x} \le \frac{1}{e}$ is bounded.
\end{proof}
\begin{proof}[Proof of \Cref{contract:sum1}]
	Following the proof of \Cref{contract:power},
	\begin{align}
	\left(\sum_{k < \tau} \norm{(I - \eta G)^k g_{i_k}}\right)^2 &\le \tau \sum_{k < \tau} \norm{(I - \eta G)^k g_{i_k}}^2 \\
	&\le n\tau \sum_{k < \tau} \sum_i \lambda_i (1 - \eta \lambda_i)^{2k} \\
	&= O\left(\frac{\tau}{\eta}\right)
	\end{align}
\end{proof}
\begin{proof}[Proof of \Cref{contract:sumsq}]
	\begin{align}
	\sum_{k < \tau} \|(I - \eta G)^k g_{i_k}\|^2 &\le \sum_{k < \tau} \tr\left[(I - \eta G)^k g_{i_k}g_{i_k}^T (I - \eta G)^k \right] \\
	&\le n \sum_{k < \tau} \tr\left[(I - \eta G)^{2k} G \right] \\
	&\le n \sum_{k < \tau} \sum_i \lambda_i (1-\eta \lambda_i)^{2k} \\
	&= O\left(\frac{1}{\eta}\right)
	\end{align}
\end{proof}
\begin{proof}[Proof of \Cref{contract:sumsqapart}]
	\begin{align}
	&\sum_{k < \tau} \|(I - \eta G)^k g_{i_k}\|\|(I - \eta G)^k g_{j_k}\| \\
	&\le \left[\sum_{k < \tau} \|(I - \eta G)^k g_{i_k}\|^2\right]^{1/2} \left[\sum_{k < \tau} \|(I - \eta G)^k g_{j_k}\|^2\right]^{1/2} \\
	&= O(1/\eta)
	\end{align}
	by \Cref{contract:sumsq}.
\end{proof}
\begin{proof}[Proof of \Cref{contract:sumG}]
	\begin{align}
	\sum_{k < \tau} \|(I - \eta G)^k G\| &\le \sum_{k \le \tau} \sum_i (I - \eta \lambda_i)^k \lambda_i \\
	&= O\left(\frac{1}{\eta}\right).
	\end{align}
\end{proof}

The following concentration inequality is from \citet{jin2019stochastic}:
\begin{lemma}[Hoeffding-type inequality for norm-subGaussian vectors]\label{lem:subgaussian}
	Given $X_1,\ldots,X_n \in \mathbb{R}^d$ and corresponding filtrations $\mathcal{F}_i = \sigma(X_1,\ldots,X_n)$ for $i \in [n]$ such that for some fixed $\sigma_1,\ldots,\sigma_n$:
	\begin{align}
		\E[X_i | F_{i-1}] = 0, \mathbb{P}[\|X_i\| \ge t | \mathcal{F}_{i-1}] \le 2e^{-\frac{t^2}{2 \sigma_i^2}},
	\end{align}
	we have that for any $\iota > 0$ there exists an absolute constant $c$ such that with probability at least $1 - 2d e^{-\iota}$,
	\begin{align}
		\norm{\sum_{i=1}^n X_i} \le c \cdot \sqrt{\sum_{i=1}^n \sigma_i^2 \cdot \iota}.
	\end{align}
\end{lemma}

\begin{lemma}\label{lem:klanalytic}
	Assume that $L$ is analytic and $\theta$ is restricted to some compact set $\mathcal{D}$. Then there exist $\delta > 0, \mu > 0, \epsilon_{KL} > 0$ such that \Cref{asm:kl} is satisfied.
\end{lemma}
\begin{proof}
	It is known that there exist $\mu_{\theta},\delta_{\theta}$ satisfying the KL-inequality in the neighborhood of any critical point $\theta$ of $L$, i.e. for every critical point $\theta$, there exists a neighborhood $U_\theta$ of $\theta$ such that for any $\theta' \in U_\theta$,
		\begin{align}
			L(\theta')-L(\theta) \le \mu_\theta\|\nabla L(\theta')\|^{1+\delta_\theta}.
		\end{align}
 Let $S = \{\theta \in \mathcal{D} : L(\theta) = L(\theta^*)\}$ for any global minimizer $\theta^*$. For every global min $\theta \in S$, let $U_{\theta}$ be a neighborhood of $\theta$ such that the KL inequality holds with constants $\mu_{\theta},\delta_{\theta}$. Because $\mathcal{D}$ is compact and $S$ is closed, $S$ is compact and there must exist some $\theta_1,\ldots,\theta_n$ such that $S \subset \bigcup_{i \in [k]} U_{\theta_i}$. Let $\delta = \min_i \delta_{\theta_i}$. Then for all $i$, there must exist some $\mu_i$ such that $\mu_i,\delta$ satisfies the KL inequality and let $\mu = \max_i \mu_i$. Finally, let $U = \bigcup_i U_{\theta_i}$ which is an open set containing $S$. Then $\mathcal{D} \setminus U$ is a compact set and therefore $L$ must achieve a minimum $\epsilon_{KL}$ on this set. Note that $\epsilon_{KL} > 0$ as $S \subset U$. Then if $L(\theta) \le \epsilon_{KL}$, $\theta \in U$ so $\mu,\delta$ satisfy the KL inequality at $\theta$.
\end{proof}
\section{Extension to SGD with Momentum}

We now prove \Cref{lem:momentumcoupling}. We will copy all of the notation from \Cref{sec:sketch:coupling}. As before we define $v_k = \theta_k - \Phi_k(\theta^*)$. Define $\xi$ by $\xi_0 = 0$ and
\begin{align}
	\xi_{k+1} = (I - \eta G)\xi_k + \epsilon_k^* + \beta(\xi_k - \xi_{k-1}).
\end{align}
We now define the following block matrices that will be crucial in our analysis:
\begin{align}
A = \begin{bmatrix}
 I - \eta G + \beta I & - \beta I \\
 1 & 0
 \end{bmatrix} \qqtext{and} J = \begin{bmatrix} I \\ 0 \end{bmatrix} \qqtext{and} B_j = J^T A^j J.
\end{align}
Then we are ready to prove the following proposition:
\begin{proposition}
	With probability $1-2de^{-\iota}$, $\|\xi_k\| \le \mathscr{X}$.
\end{proposition}
\begin{proof}
	Define $\bar \xi_k = \cvec{\xi_k}{\xi_{k-1}}$. Then the above can be written as:
	\begin{align}
		\bar \xi_{k+1} = A \bar \xi_k + J\epsilon_k^*
	\end{align}
	Therefore by induction,
	\begin{align}
		\bar \xi_k = \sum_{j < k} A^{k-j-1} J\epsilon_j^* \implies \xi_k = \sum_{j < k} B_{k-j-1}\epsilon_j^*.
	\end{align}
	The partial sums form a martingale and by \Cref{prop:momentumlimit}, the quadratic covariation is bounded by
	\begin{align}
		(1-\beta) n\eta \lambda \sum_{j=0}^\infty B_j G B_j^T \preceq \frac{n\lambda}{\nu} \Pi_G
	\end{align}
	so by \Cref{cor:azumacov} we are done.
\end{proof}

We will prove \Cref{lem:momentumcoupling} by induction on $t$. Assume that $\|r_k\| \le \mathscr{D}$ for $k \le t$. First, we have the following version of \Cref{prop:sketch:taylor2}:

\begin{proposition} Let $\bar r_k = \cvec{r_k}{r_{k-1}}$. Then,
	\begin{align}
		\bar r_{k+1} 
	&= A \bar r_k + J\left(- \eta \left[\frac{1}{2}\nabla^3 L(\xi_k,\xi_k) - \lambda \nabla R \right] + m_k + z_k + O(\eta \mathscr{X}(\sqrt{\mathscr{L}} + \mathscr{M} + \mathscr{X}^2))\right)
	\end{align}
\end{proposition}
\begin{proof}
	As before we have that
	\begin{align}
		v_{k+1} &= (I - \eta G)v_k - \eta \left[\frac{1}{2}\nabla^3 L(\xi_k,\xi_k) - \lambda \nabla R \right] + \epsilon_k^* + m_k + z_k \notag \\
		&\qquad + O(\eta \mathscr{X}(\sqrt{\mathscr{L}} + \mathscr{M} + \mathscr{X}^2)) + \beta(v_k - v_{k-1})
	\end{align}
	and subtracting the definition of $\xi_k$ proves the top block of the proposition. The bottom block is equivalent to the identity $r_k = r_k$.
\end{proof}

\begin{proposition}
	\begin{align}
	r_{t+1} = -\eta \sum_{k \le t} B_{t-k} \left[\frac{1}{2}\nabla^3 L(\xi_k,\xi_k) - \lambda \nabla R\right] + O\left(\sqrt{\eta \lambda t}\left(\sqrt{\mathscr{L}}+\mathscr{X}\right) + \eta t \mathscr{X}\left(\sqrt{\mathscr{L}} + \mathscr{M} + \mathscr{X}^2\right)\right).
	\end{align}
\end{proposition}
\begin{proof}
	We have from the previous proposition that
	\begin{align}
	\bar r_{t+1} = \sum_{k \le t} A^{t-k} J \left(- \eta \left[\frac{1}{2}\nabla^3 L(\xi_k,\xi_k) - \lambda \nabla R \right] + m_k + z_k + O(\eta \mathscr{X}(\sqrt{\mathscr{L}} + \mathscr{M} + \mathscr{X}^2))\right)	
	\end{align}
	so
	\begin{align}
	r_{t+1} = \sum_{k \le t} B_{t-k} \left(- \eta \left[\frac{1}{2}\nabla^3 L(\xi_k,\xi_k) - \lambda \nabla R \right] + m_k + z_k + O(\eta \mathscr{X}(\sqrt{\mathscr{L}} + \mathscr{M} + \mathscr{X}^2))\right).
	\end{align}
	By \Cref{cor:Bbound}, we know that $B_k$ is bounded by $\frac{1}{1-\beta}$ so the remainder term is bounded by $O(\eta t \mathscr{X}(\sqrt{\mathscr{L}} + \mathscr{M} + \mathscr{X}^2))$. Similarly, by the exact same concentration inequalities used in the proof of \Cref{prop:sketch:OUcov}, we have that the contribution of the $m_k,z_k$ terms is at most $O\left(\sqrt{\eta \lambda t}\left(\sqrt{\mathscr{L}}+\mathscr{X}\right)\right)$ which completes the proof.
\end{proof}

\begin{proposition}
	\begin{align}
		\eta \sum_{k \le t} B_{t-k} \left[\frac{1}{2}\nabla^3 L(\xi_k,\xi_k) - \lambda \nabla R\right] = O\left(\sqrt{\eta t} \mathscr{X}^2 + \eta t \mathscr{X} \sqrt{\mathscr{L}}\right).
	\end{align}
\end{proposition}
\begin{proof}
	As in the proof of \Cref{prop:sketch:OUcov}, we define
	\begin{align}
		S^* = \lambda\left(2-\frac{\eta}{1+\beta}\nabla^2 L\right)^{-1},\qquad \bar S = \lambda\left(2-\frac{\eta}{1+\beta} G\right)^{-1}, \qqtext{and} S_k = \xi_k \xi_k^T.
	\end{align}
	Then note that $\nabla R = \frac{1}{2}\nabla^3 L(S^*)$ so it suffices to bound
	\begin{align}
		\eta \sum_{k \le t} B_{t-k} \nabla^3 L(S_k - S^*).
	\end{align}
	As before we can decompose this as
	\begin{align}
		\eta \sum_{k \le t} B_{t-k} \nabla^3 L(S_k - \bar S) + \eta \sum_{k \le t} B_{t-k} \nabla^3 L(\bar S - S^*).
	\end{align}
	We will begin by bounding the second term. Note that
	\begin{align}
		\eta \sum_{k \le t} B_{t-k} \nabla^3 L(\bar S - S^*) = O(\eta \|\bar S - S^*\|).
	\end{align}
	We can rewrite this as
	\begin{align}
		S^* - \bar S = \lambda \left[(2-\eta \nabla^2 L)^{-1}\left((2-\eta G)-(2-\eta \nabla^2 L)\right)(2-\eta G)^{-1}\right] = O(\eta \lambda \sqrt{\mathscr{L}})
	\end{align}
	so this difference contributes at most $O(\eta^2 \lambda t \sqrt{\mathscr{L}}) = O(\eta t \mathscr{X} \sqrt{\mathscr{L}})$. For the first term, let $D_k = S_k - \bar S$. We will decompose $\nabla^3 L$ as before to get
	\begin{align}
		\frac{1}{n} \sum_{i=1}^n \eta \sum_{k \le t} B_{t-k} \left[H_i D_k g_i + \frac{1}{2} g_i \tr(D_k H_i) + O(\sqrt{\mathscr{L}}\mathscr{X}^2)\right].
	\end{align}
	The third term can be bound by the triangle inequality by \Cref{cor:Bbound} to get $O(\eta t \sqrt{\mathscr{L}}\mathscr{X}^2)$. The second term can be bound by \Cref{prop:momentum:contractsum1} to get $O(\sqrt{\eta t}\mathscr{X}^2)$. 
	
	The final remaining term is the first term. Define
	\begin{align}
	\bar S' = \lambda \begin{bmatrix} \bar S & (I - \frac{\eta}{1+\beta}G) \bar S \\ (I - \frac{\eta}{1+\beta}G) \bar S & \bar S \end{bmatrix}.
	\end{align}
	From the proof of \Cref{prop:momentumlimit}, we can see that $\bar S'$ satisfies
	\begin{align}
		\bar S' = A \bar S' A^T + (1-\beta)\eta \lambda J G J^T.
	\end{align}
	We also have:
	\begin{align}
		\bar \xi_{k+1} = A \bar \xi_k + J \epsilon_k^*
	\end{align}
	so
	\begin{align}
		\bar \xi_{k+1}\bar \xi_{k+1}^T = A \bar \xi_k \bar \xi_k^T A^T + J \epsilon_k^* \bar \xi_k^T A^T + A \bar \xi_k (\epsilon_k^*)^T J^T + J \epsilon_k^* \epsilon_k^* J^T.
	\end{align}
	Let $D_k' = \bar \xi_k \bar \xi_k^T - \bar S'$. Then,
	\begin{align}
		D_{k+1}' = AD_k' A^T + W_k + Z_k
	\end{align}
	where $W_k = J \epsilon_k^* \bar \xi_k^T A^T + A \bar \xi_k (\epsilon_k^*)^T J^T$ and $Z_k = J [\epsilon_k^* \epsilon_k^*-(1-\beta)\eta \lambda G] J^T$.
	Then,
	\begin{align}
		D_k' =  A^k \bar S' A^k + \sum_{j < k} A^{k-j-1} [W_j + Z_j] (A^T)^{k-j-1}
	\end{align}
	so
	\begin{align}
		D_k =  J^T A^k \bar S' A^k J + \sum_{j < k} J^T A^{k-j-1} [W_j + Z_j] (A^T)^{k-j-1} J.
	\end{align}
	Plugging this into the first term, which we have not yet bounded, we get
	\begin{align}
		\frac{1}{n} \sum_{i=1}^n \eta \sum_{k \le t} B_{t-k} H_i \left[J^T A^k \bar S' A^k J + \sum_{j < k} J^T A^{k-j-1} [W_j + Z_j] (A^T)^{k-j-1} J\right] g_i.
	\end{align}
	For the first term in this expression we can use \Cref{prop:momentum:contractsum1} to bound it by $O(\sqrt{\eta t} \lambda) \le O(\sqrt{\eta t}\mathscr{X}^2)$. Therefore we are just left with the second term. Changing the order of summation gives
	\begin{align}
		\eta\frac{1}{n} \sum_{i=1}^n \sum_{j \le t} \sum_{k = j+1}^t B_{t-k} H_i J^T A^{k-j-1} (W_j + Z_j) (A^T)^{k-j-1} J g_i.
	\end{align}
	Recall that $\epsilon_j^* = \frac{\eta}{B} \sum_{l \in \mathcal{B}^{(j)}} \epsilon^{(j)}_l g_l$. First, isolating the inner sum for the $W$ term, we get
	\begin{align}
		&\sum_{k = j+1}^t B_{t-k} H_i J^T A^{k-j} \bar \xi_j (\epsilon_j^*)^T J^T A^{k-j-1} J g_i \\
		&\qquad+ \sum_{k = j+1}^t B_{t-k} H_i J^T A^{k-j-1} J \epsilon_j^* \bar \xi_j^T A^{k-j} J g_i. \notag \\
		&= \frac{\eta}{B} \sum_{l \in \mathcal{B}^{(j)}} \epsilon^{(j)}_l \Bigl[\sum_{k = j+1}^t B_{t-k} H_i J^T A^{k-j} \bar \xi_j g_l^T B_{k-j-1} g_i \label{eq:momentum:OUcov:W}\\
		&\qquad+ \sum_{k = j+1}^t B_{t-k} H_i B_{k-j-1} g_l \bar \xi_j^T A^{k-j} J g_i\Bigr] \notag.
	\end{align}
	The inner sums are bounded by $O(\mathscr{X}\eta^{-1})$ by \Cref{prop:momentum:contractsumapart}. Therefore by \Cref{lem:azuma}, with probability at least $1-2de^{-\iota}$, the contribution of the $W$ term in \Cref{eq:wjzj} is at most $O(\sqrt{\eta \lambda k \iota}\mathscr{X}) = O(\sqrt{\eta k}\mathscr{X}^2)$. The final remaining term to bound is the $Z$ term in \eqref{eq:wjzj}. We can write the inner sum as
	\begin{align}
		\frac{\eta\lambda(1-\beta)}{B^2}\sum_{k = j+1}^t B_{t-k} H_i J^T A^{k-j-1} J \left(\frac{1}{\sigma^2}\sum_{l_1,l_2 \in \mathcal{B}^{(k)}} \epsilon^{(j)}_{l_1} \epsilon^{(j)}_{l_2} g_{l_1} g_{l_2}^T  - G\right) B_{k-j-1} g_i
	\end{align}
	which by \Cref{prop:momentum:contractsumapart} is bounded by $O(\lambda)$. Therefore by \Cref{lem:azuma}, with probability at least $1-2de^{-\iota}$, the full contribution of $Z$ is $O(\eta \lambda \sqrt{t \iota}) = O(\sqrt{\eta t}\mathscr{X}^2)$.
\end{proof}

Putting all of these bounds together we get with probability at least $1-10de^{-\iota}$,
	\begin{align}
		\|r_{t+1}\| &= O\left[\sqrt{\eta \mathscr{T}}\mathscr{X}(\sqrt{\mathscr{L}} + \mathscr{X}) + \eta \mathscr{T} \mathscr{X}(\sqrt{\mathscr{L}} + \mathscr{M} + \mathscr{X}^2)\right] \\
		&= O\left(\frac{\lambda^{1/2+\delta/2}\iota}{\sqrt{c}}\right) \le \mathscr{D}
	\end{align}
	for sufficiently large $c$ which completes the induction.

\subsection{Momentum Contraction Bounds}
Let $u_i,\lambda_i$ be the eigenvectors and eigenvalues of $G$. Consider the basis $\bar U$ of $\mathbb{R}^{2d}$: $[u_1,0],[0,u_1],\ldots,[u_d,0],[0,u_d]$. Then in this basis, $A,J$ are block diagonal matrix with $2 \times 2$ and $2 \times 1$ diagonal blocks:
\begin{align}
A_i = \begin{bmatrix}
 1 - \eta \lambda_i	+ \beta & -\beta \\ 1 & 0
 \end{bmatrix} \qqtext{and} J_i = \begin{bmatrix} 1 \\ 0 \end{bmatrix}.
\end{align}
Let the eigenvalues of $A_i$ be $a_i,b_i$ so
\begin{align}
	a_i = \frac{1}{2} \left(1 - \eta \lambda_i + \beta +\sqrt{(1 - \eta \lambda_i + \beta)^2-4 \beta }\right) && b_i = \frac{1}{2} \left(1 - \eta \lambda_i + \beta -\sqrt{(1 - \eta \lambda_i + \beta)^2-4 \beta }\right).
\end{align}
Note that these satisfy $a_i + b_i = 1 - \eta \lambda_i + \beta$ and $a_i b_i = \beta$.
\begin{proposition}
If $\eta \in (0,\frac{2(1+\beta)}{\ell})$, then $\rho(A_i) \le 1$. If $\lambda_i \ne 0$ then $\rho(A_i) < 1$.
\end{proposition}
\begin{proof}
	First, if $(1 - \eta \lambda_i + \beta)^2 - 4B \le 0$ then $\abs{a_i} = \abs{b_i} = \sqrt{\beta} < 1$ so we are done. Otherwise, we can assume WLOG that $\eta \lambda_i < 1+\beta$ because $\rho(A_i)$ remains fixed by the transformation $\eta \lambda_i \to 2(1+\beta) - \eta\lambda_i$. Then $a_i > b_i > 0$ so it suffices to show $a_i < 1$. Let $x = 1 - \eta \lambda_i + \beta$. Then,
	\begin{align}
		\frac{x + \sqrt{x^2-4\beta}}{2} < 1 &\iff \sqrt{x^2-4\beta} < 2-x \iff x < 1+\beta.
	\end{align}
	and similarly for $\le$ in place of $<$ so we are done.
\end{proof}

\begin{proposition} Let $s_k = \sum_{j < k} a_i^{k-j-1} b_i^j$. Then,
\begin{align}
	A_i^k = \begin{bmatrix} s_{k+1} & -\beta s_k \\ s_k & -\beta s_{k-1} \end{bmatrix}.
\end{align}	
\end{proposition}
\begin{proof}
	We proceed by induction on $k$. The base case is clear as $s_2 = a_i + b_i = 1 - \eta \lambda_i + \beta$, $s_1 = 1$, and $s_0 = 0$. Now assume the result for some $k \ge 0$. Then,
	\begin{align}
		A_i^{k+1} = \begin{bmatrix} s_{k+1} & -\beta s_k \\ s_k & -\beta s_{k-1} \end{bmatrix} \begin{bmatrix} a_i + b_i & -\beta \\ 1 & 0 \end{bmatrix} = \begin{bmatrix} s_{k+1} & -\beta s_k \\ s_k & -\beta s_{k-1} \end{bmatrix}
	\end{align}
	because $(a_i + b_i) s_k - \beta s_{k-1} = (a_i + b_i) s_k - a_ib_i s_{k-1} = s_{k+1}$.
\end{proof}
\begin{proposition}
	\begin{align}
		\abs{J_i^T A_i^k J_i} \le \frac{1}{1-\beta}.
	\end{align}
\end{proposition}
\begin{proof}
	From the above proposition,
	\begin{align}
		\abs{J_i^T A_i^k J_i} \le \sup_k \abs{s_{k+1}}.
	\end{align}
	Then for any $k$,
	\begin{align}
		\abs{s_{k+1}} = \abs{\sum_{j \le k} a_i^{k-j} b_i^j} \le \sum_{j \le k} \abs{a_i}^{k-j}\abs{b_i}^j \le \sum_{j \le k} \abs{a_i}^j \abs{b_i}^j = \sum_{j \le k} \beta^j \le \frac{1}{1-\beta}.
	\end{align}
	where the second inequality follows from the rearrangement inequality as $\{\abs{a_i}^{k-j}\}_j$ is an increasing sequence and $\{\abs{b_i}^j\}_j$ is a decreasing sequence.
\end{proof}

\begin{corollary}\label{cor:Bbound}
	\begin{align}
		\|B_k\|_2 \le \frac{1}{1-\beta}.
	\end{align}
\end{corollary}

\begin{proposition}\label{prop:momentumlimit}
	\begin{align}
	\sum_{j=0}^\infty B_j G B_j^T = \frac{1}{\eta (1-\beta)} \Pi_G \left(2-\frac{\eta}{(1+\beta)}G\right)^{-1}.
	\end{align}
\end{proposition}
\begin{proof} Consider $\sum_{j=0}^\infty A^j J G J^T (A^T)^j$. We will rewrite this expression in the basis $\bar U$. Then the $i$th diagonal block will be equal to
	\begin{align}
		\lambda_i \sum_{j=0}^\infty A_i^j J_i \lambda_i J_i^T (A_i^T)^j = \lambda_i \sum_{j=0}^\infty \begin{bmatrix}
 s_{j+1}^2 & s_j s_{j+1} \\ s_j s_{j+1} & s_j^2	
 \end{bmatrix}.
	\end{align}
If $\lambda_i = 0$ then this term is $0$. Otherwise, we know that $\abs{a_i},\abs{b_i} < 1$ so this infinite sum converges to some matrix $S = \begin{bmatrix}s_{11} & s_{12} \\ s_{21} & s_{22} \end{bmatrix}$. Then plugging this into the fixed point equation gives
\begin{align}
	S_i = A_i S_i A_i^T + J_i \lambda_i J_i^T
\end{align}
and solving this system entry wise for $s_{11},s_{12},s_{21},s_{22}$ gives
\begin{align}
	S_i = \frac{1}{\eta (1-\beta)} \begin{bmatrix}
		\frac{1}{2-\frac{\eta}{1+\beta}\lambda_i} & \frac{1+\beta-\eta\lambda_i}{2(1+\beta)-\eta \lambda_i} \\ \frac{1+\beta-\eta\lambda_i}{2(1+\beta)-\eta \lambda_i} & \frac{1}{2-\frac{\eta}{1+\beta}\lambda_i}.
	\end{bmatrix}	
\end{align}
Converting back to the original basis gives the desired result.
\end{proof}

\begin{proposition}\label{prop:momentum:contractsum1}
	\begin{align}\sum_{k < \tau} \|A^k J g_i\| = O\left(\sqrt{\frac{\tau}{\eta}}\right)	
	\end{align}
\end{proposition}
\begin{proof}
	By Cauchy we have that
	\begin{align}
	\left(\sum_{k < \tau} \|A^k J g_i\|\right)^2 &\le \tau \sum_{k < \tau} \|A^k J g_i\|^2 \\
	&\le \tau \sum_{k < \tau} \tr[A^k J G J^T (A^k)^T] \\
	&\le O\left(\frac{\tau}{\eta}\right)
	\end{align}
	by \Cref{prop:momentumlimit}.
\end{proof}

\begin{proposition}\label{prop:momentum:contractsum2}
	\begin{align}
		\sum_{k < \tau} \|A^k J g_{i_k}\|^2 = O\left(\frac{1}{\eta}\right).
	\end{align}
\end{proposition}
\begin{proof}
	\begin{align}
		\sum_{k < \tau} \|A^k J g_i\|^2 \le \tau \sum_{k < \tau} \tr[A^k J G J^T (A^k)^T] \le O\left(\sqrt{\frac{1}{\eta}}\right).
	\end{align}
\end{proof}

\begin{proposition}\label{prop:momentum:contractsumapart}
	\begin{align}
		\sum_{k < \tau} \|A^k J g_{i_k}\|\|A^k J g_{j_k}\| = O\left(\frac{1}{\eta}\right).
	\end{align}
\end{proposition}
\begin{proof}
	\begin{align}
		\left(\sum_{k < \tau} \|A^k J g_{i_k}\|\|A^k J g_{j_k}\|\right)^2 \le \left(\sum_{k \le \tau}\|A^k J g_{i_k}\|^2 \right)\left(\sum_{k \le \tau}\|A^k J g_{j_k}\|^2 \right) = O(1/\eta^2)
	\end{align}
	by \Cref{prop:momentum:contractsum2}.
\end{proof}

\end{document}